\documentclass[10pt,journal,compsoc]{IEEEtran}

\ifCLASSOPTIONcompsoc
  \usepackage[nocompress]{cite}
\else
  \usepackage{cite}
\fi

\usepackage{graphicx}
\usepackage{subfigure}
\usepackage[numbers,square]{natbib}
\graphicspath{{./images/exemplar_cnn/},{./images/matching_new/}}
\usepackage{multirow}
\usepackage{rotating}
\usepackage{paralist}
\usepackage{hhline}
\usepackage{amsmath}
\usepackage{amssymb}
\usepackage{amsfonts}
\usepackage{color}
\usepackage{wrapfig}
\usepackage{float}

\usepackage{color}
\definecolor{dgreen}{rgb}{0,.7,0}
\definecolor{dyellow}{rgb}{.7,.7,0}
\definecolor{dred}{rgb}{.7,0,0}
\definecolor{dblue}{rgb}{0,0,0.7}
\definecolor{dmag}{rgb}{.6,0,0.6}

\newcommand{\norm}[1]{\Vert #1 \Vert}

\newcommand{\grad}{\nabla}

\newcommand{\leadingzero}[1]{\ifnum #1<10 0\the#1\else\the#1\fi}
\newcommand{\todayIII}{\leadingzero{\day}/\leadingzero{\month}/\the\year}

\newcommand{\cT}{\mathcal{T}}
\newcommand{\alphaset}{\mathcal{A}}
\newcommand{\me}{\mathbb{E}}
\newcommand{\scal}[2]{\langle #1,\, #2 \rangle}
\newcommand{\softmax}{\mathrm{softmax} \,}
\newcommand{\avgg}[1][]{\widehat{\mathbf{g}_{#1}}}

\newcommand{\bx}{\mathbf{x}}
\newcommand{\be}{\mathbf{e}}
\newcommand{\balpha}{\mathbf{\alpha}}
\newcommand{\bz}{\mathbf{z}}
\newcommand{\by}{\mathbf{y}}
\newcommand{\bW}{\mathbf{W}}
\newcommand{\nll}{M}
\newcommand{\oR}{\mathbb{R}}
\newcommand{\linspan}{span\,}
\newcommand{\bu}{\mathbf{u}}
\newcommand{\bone}{\mathbf{1}}
\newcommand{\diag}{diag\,}

\newtheorem{proposition}{Proposition}
\newenvironment{proof}[1][Proof]{\begin{trivlist}
\item[\hskip \labelsep {\bfseries #1}]}{\end{trivlist}}

\begin{document}
%
\title{Discriminative Unsupervised Feature Learning with Exemplar Convolutional Neural Networks}
%
%

\author{Alexey~Dosovitskiy,
        Philipp~Fischer,
        Jost~Tobias~Springenberg,
        Martin~Riedmiller,
        Thomas~Brox
\IEEEcompsocitemizethanks{\IEEEcompsocthanksitem All authors are with the Computer Science Department\protect\\ at the University of Freiburg\protect\\
E-mail: \{dosovits, fischer, springj, riedmiller, brox\}@cs.uni-freiburg.de}
\thanks{}}

\IEEEtitleabstractindextext{%
\begin{abstract}
Deep convolutional networks have proven to be very successful in learning task specific features that allow for unprecedented performance on various computer vision tasks.
Training of such networks follows mostly the supervised learning paradigm, where sufficiently many input-output pairs are required for training.
Acquisition of large training sets is one of the key challenges, when approaching a new task.
In this paper, we aim for generic feature learning and present an approach for training a convolutional
network using only unlabeled data.
To this end, we train the network to discriminate between a set of surrogate classes.
Each surrogate class is formed by applying a variety of transformations to
a randomly sampled 'seed' image patch.
In contrast to supervised network training, the resulting feature representation is not class specific.
It rather provides robustness to the transformations that have been applied during training.
This generic feature representation allows for classification results that outperform the
state of the art for unsupervised learning on several popular datasets (STL-10, CIFAR-10, Caltech-101, Caltech-256). While such generic features cannot compete with class specific features from supervised training on a classification task, we show that they are advantageous on geometric matching problems, where they also outperform the SIFT descriptor.
\end{abstract}
}

\maketitle

\IEEEdisplaynontitleabstractindextext

%
\IEEEpeerreviewmaketitle

\IEEEraisesectionheading{\section{Introduction}}


In the recent two years Convolutional Neural Networks (CNNs) trained in a supervised manner via backpropagation dramatically improved the state of the art performance on a variety of Computer Vision tasks, such as image classification~\citep{Krizhevsky_NIPS2012, Zeiler_ECCV2014, Donahue_ICML2014,Razavian_CVPR2014},  detection~\citep{Girshick_CVPR2014, Sermanet_ICLR2014}, semantic segmentation~\cite{Hariharan_CVPR2015, Long_CVPR2015}. Interestingly, the features learned by such networks often generalize to new datasets: for example, the feature representation of a network trained for classification on ImageNet~\cite{imagenet} also performs well on PASCAL VOC \cite{pascal}. Moreover, a network can be adapted to a new task by replacing the loss function and possibly the last few layers of the network and 
\emph{fine-tuning} it to the new problem, i.e. adjusting the weights using backpropagation. With this approach, typically much smaller training sets are sufficient. 

Despite the big success of this approach, it has at least two potential drawbacks. First, there is the need for huge labeled datasets to be used for the initial supervised training. These are difficult to collect, and there are diminishing returns of making the dataset larger and larger. Hence, unsupervised feature learning, which has quick access to arbitrary amounts of data, is conceptually of large interest despite its limited performance so far. Second, although the CNNs trained for classification generalize well to similar tasks, such as object class detection, semantic segmentation, or image retrieval, the transfer becomes less efficient the more the new task differs from the original training task. In particular, object class annotation may not be beneficial to learn features for class-independent tasks, such as descriptor matching.

In this work, we propose a procedure for training a CNN that does not rely on any labeled data but rather makes use of a surrogate task automatically generated from unlabeled images. The surrogate task is designed to yield generic features that are descriptive and robust to typical variations in the data. The variation is simulated by randomly applying transformations to a 'seed' image. This image and its transformed versions constitute a surrogate class. In contrast to previous data augmentation approaches, only a single seeding sample is needed to build such a class. Consequently, we call thus trained networks \emph{Exemplar-CNN}.

By construction, the representation learned by the Exemplar-CNN is discriminative, while also invariant to some typical transformations. These properties make it useful for various vision tasks. We show that the  feature representation learned by the Exemplar-CNN performs well on two very different tasks: object classification and descriptor matching. The classification accuracy obtained with the Exemplar-CNN representation exceeds that of all previous unsupervised methods on four benchmark datasets: STL-10, CIFAR-10, Caltech-101, Caltech-256. On descriptor matching, we show that the feature representation outperforms the representation of the AlexNet~\cite{Krizhevsky_NIPS2012}, which was trained in a supervised, class-specific manner on ImageNet. Moreover, it outperforms the popular SIFT descriptor. 

\begin{figure*}
\centering
\begin{minipage}[t]{.5\textwidth}
  \centering
  \includegraphics[width=.95\linewidth]{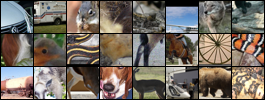}
  \begin{minipage}[t]{.9\linewidth}
  \vskip -0.1in
  \caption{Exemplary patches sampled from the STL unlabeled dataset which are
    later augmented by various transformations to obtain surrogate data
    for the CNN training.}
  \label{fig:sample_patches}
  \end{minipage}
\end{minipage}%
\begin{minipage}[t]{.5\textwidth}
  \centering
  \includegraphics[width=.95\linewidth]{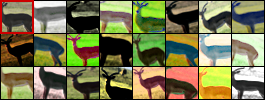}
  \begin{minipage}[t]{.9\linewidth}
  \vskip -0.1in
  \caption{Several random transformations applied to one of the patches
    extracted from the STL unlabeled dataset. The original ('seed') patch is in the
    top left corner.}
  \label{fig:sample_transformations}
  \end{minipage}
\end{minipage}
  \vskip -0.1in
\end{figure*}

\subsection{Related Work}
Our approach is related to a large body of work on unsupervised
learning of invariant features and training of convolutional neural networks.

Convolutional training is commonly used in both
supervised and unsupervised methods to utilize the invariance of image statistics to
translations
\citep{LeCun_NC1989,Kavukcuoglu_NIPS2010,Krizhevsky_NIPS2012}. Similar
to our approach, most successful methods employing
convolutional neural networks for object recognition rely on data
augmentation to generate additional training samples for their
classification objective \citep{Krizhevsky_NIPS2012,Zeiler_ECCV2014}. While
we share the architecture (a convolutional neural network) with
these approaches, our method does not rely on any labeled training
data.

In unsupervised learning, several studies on learning invariant
representations exist. Denoising autoencoders
\citep{Vincent_ICML2008}, for example, learn features that are robust
to noise by trying to reconstruct data from randomly perturbed input
samples.
Zou et al. \citep{Zou_NIPS2012} learn invariant features
from video by enforcing a temporal slowness constraint on the feature
representation learned by a linear autoencoder.
Sohn et al. \citep{Sohn_ICML2012} and Hui et al. \citep{Hui_ICML2013} learn features
invariant to local image transformations. In contrast to our
discriminative approach, all these methods rely on directly modeling
the input distribution and are typically hard to use for jointly
training multiple layers of a CNN.

The idea of learning features that are invariant to transformations
has also been explored for supervised training of neural
networks. The research most similar to ours is early work on tangent
propagation \citep{Simard_NIPS1992} (and the related double
backpropagation \citep{Drucker_TNN1992}) which aims to learn invariance to
small predefined transformations in a neural network by directly penalizing the derivative
of the output with respect to the magnitude of the
transformations. In contrast, our algorithm does not regularize the derivative explicitly. Thus it is less sensitive to the magnitude of the applied transformation.

This work is also loosely related to the use of unlabeled data for
regularizing supervised algorithms, for
example self-training~\citep{Amini_ECAI2002} or entropy
regularization~\citep{Grandvalet_SSL2006}. In
contrast to these semi-supervised methods, Exemplar-CNN training does not require any labeled data.


Finally, the idea of creating an auxiliary task in order to learn a good data representation was used in~\citep{Ahmed_ECCV2008, Collobert_JMLR2011}.


\section{Creating Surrogate Training Data}

The input to the proposed training procedure is a set of unlabeled images, which come
from roughly the same distribution as the images in which we later aim to
compute the learned features. We randomly sample $N$ patches of size $32
\times 32$ pixels from different images at varying positions and
scales forming the initial training set $X = \{\bx_1, \ldots \bx_N\}$. We are interested in patches containing objects or parts of objects, hence we sample only from regions containing considerable gradients. More precisely, we sample a patch with probability proportional to mean squared gradient magnitude within the patch. Exemplary patches sampled from the STL-10 unlabeled
dataset are shown in Fig.~\ref{fig:sample_patches}.

We define a family of transformations $\{T_\balpha|\, \balpha \in \alphaset\}$ parameterized by
vectors $\balpha \in \alphaset$, where $\alphaset$ is the set of all possible parameter vectors. Each transformation $T_\balpha$ is a \emph{composition} of elementary transformations. To learn features for the purpose of object classification, we used transformations from the following list:
\begin{compactitem}
 \item translation: vertical and horizontal translation by a distance within $0.2$ of the patch size;
 \item scaling: multiplication of the patch scale by a factor between $0.7$ and $1.4$;
 \item rotation: rotation of the image by an angle up to $20$ degrees;
 \item contrast 1: multiply the projection of each patch pixel onto the principal components of the set of all pixels by a factor between $0.5$ and $2$ (factors are independent for each principal component and the same for all pixels within a patch);
 \item contrast 2: raise saturation and value (S and V components of the HSV color representation) of all pixels to a power between $0.25$ and $4$ (same for all pixels within a patch), multiply these values by a factor between $0.7$ and $1.4$, add to them a value between $-0.1$ and $0.1$;
 \item color: add a value between $-0.1$ and $0.1$ to the hue (H
   component of the HSV color representation) of all pixels in the
   patch (the same value is used for all pixels within a patch).
\end{compactitem}
The approach is flexible with regard to extending this list by other transformations in order to serve other applications of the learned features better. For instance, in Section~\ref{sect:matching} we show that descriptor matching benefits from adding a blur transformation.

All numerical parameters of elementary transformations, when concatenated
together, form a single parameter vector $\alpha$.
For each initial patch $\bx_i \in X$ we sample $K$
random parameter vectors $\{\balpha_i^1, \ldots, \balpha_i^K\}$  and
apply the corresponding transformations $\cT_i = \{ T_{\balpha_i^1},
\ldots, T_{\balpha_i^K}\}$ to the patch $\bx_i$. This yields the set
of its transformed versions $S_{\bx_i} = \cT_i \bx_i = \{T \bx_i|\, T
\in \cT_i\}$. An example of such a set is shown in Fig.~\ref{fig:sample_transformations}~. Afterwards we subtract the mean of each pixel over the whole resulting dataset. We do not apply any other preprocessing.

\section{Learning Algorithm}

Given the sets of transformed image patches, we declare each
of these sets to be a class by assigning label $i$ to the class
$S_{x_i}$. We train a CNN to discriminate
between these surrogate classes. Formally, we minimize the following
loss function:
\begin{equation}
\label{eq:objective}
 L(X) = \sum\limits_{\bx_i \in X} \sum\limits_{T \in \cT_i} l(i,\, T \bx_i),
\end{equation}
where $l(i,\, T \bx_i)$ is the loss on the transformed sample
$T \bx_i$ with (surrogate) true label $i$. We use a
CNN with a fully connected classification layer and a softmax output layer and we optimize the multinomial negative log
likelihood of the network output, hence in our case
\begin{align}
\label{eq:objective_details}
  \begin{split}
 l(i,\, T \mathbf{x}_i) &= \nll(\be_i, f(T \mathbf{x}_i)), \\
 \nll(\by, \mathbf{f}) &= -\scal{\by}{\log \mathbf{f}} = - \sum\limits_{k} y_k \log f_k,
  \end{split}
\end{align}
where $f(\cdot)$ denotes the function computing the values of the output
layer of the CNN given the input data, and $\mathbf{e}_i$ is the
$i$th standard basis vector. We note that in the limit of an infinite
number of transformations per surrogate class, the objective function~\eqref{eq:objective} takes the form
\begin{equation}
\label{eq:formal_objective}
 \widehat{L} (X) = \sum\limits_{\mathbf{x}_i \in X} \me_\balpha [l(i,\, T_\balpha \mathbf{x}_i)],
\end{equation}
which we shall analyze in the next section.

Intuitively, the classification problem described above serves to ensure that different input samples can be distinguished. At the same time, it enforces invariance to the specified transformations. In the following
sections we provide a foundation for this intuition. We first
present a formal analysis of the objective, separating it into a well
defined classification problem and a regularizer that enforces
invariance (resembling the analysis in~\citep{Wagner_NIPS2013}). We then discuss the derived properties of this classification problem and compare it to common practices for unsupervised feature learning.

\subsection{Formal Analysis}

We denote by $\balpha \in \alphaset$ the random vector of transformation parameters, by $g(\bx)$ the vector
of activations of the second-to-last layer of the network when
presented the input patch $\bx$, by $\bW$ the matrix of the
weights of the last network layer, by $h(\bx) = \bW g(\bx)$ the last layer activations before applying the softmax,
and by $f(\bx) = \softmax(h(\bx))$ the output of the network. By plugging in the definition of the softmax activation function
\begin{equation}
 \softmax(\bz) = \exp(\bz) / \norm{\exp(\bz)}_1
\end{equation}
the objective function~\eqref{eq:formal_objective} with loss~\eqref{eq:objective_details} takes the form
\begin{equation}
\sum\limits_{\bx_i \in X} \me_\balpha \bigl[
-\scal{\be_i}{h(T_\balpha \bx_i)} + \log \norm{\exp
  (h(T_\balpha \bx_i))}_1 \bigr].
\label{eq:objective-with-softmax}
\end{equation}
With $\avgg[i] = \me_\balpha \left[ g(T_\balpha \bx_i) \right]$ being the average feature representation of transformed versions of the image patch $\bx_i$
we can rewrite Eq. \eqref{eq:objective-with-softmax} as
\begin{align}
  \begin{split}
&\sum\limits_{\bx_i \in X} \bigl[- \scal{\be_i}{\bW \avgg[i]} + \log \norm{\exp (\bW \avgg[i])}_1\bigr]\\
+ &\sum\limits_{\bx_i \in X} \bigl[ \me_\balpha \left[\log \norm{\exp
    (h(T_\balpha \bx_i))}_1 \right] - \log \norm{\exp (\bW \avgg[i])}_1
\bigr].
\end{split}
\label{eq:objective-rewritten}
\end{align}
The first sum is the objective function of a multinomial logistic regression
problem with input-target pairs $(\avgg[i],\, \be_i)$. This objective falls back to the
transformation-free instance classification problem $\overline{L}(X) = \sum_{\bx_i \in X} l(i,\, \bx_i)$
if $g(\bx_i) = \me_\balpha[g(T_{\balpha} \bx)]$. In general, this
equality does not hold and thus the first sum enforces correct
classification of the average representation $\me_\balpha[g(T_{\balpha} \bx_i)]$ for a given input
sample.
For a truly invariant representation, however, the equality is achieved.
Similarly, if we suppose that $T_\balpha \bx = \bx$ for $\balpha = 0$, that for small
values of~$\balpha$ the feature representation $g(T_\balpha \bx_i)$ is approximately linear
with respect to $\balpha$ and that the random
variable $\balpha$ is centered, i.e. $\me_\balpha \left[ \balpha \right]
= 0$, then $\avgg[i] = \me_\balpha \left[ g(T_\balpha \bx_i)
\right] \approx \me_\balpha \left[g(\bx_i) + \left.\grad_\balpha
    (g(T_\balpha \bx_i))\right|_{\balpha=0} \, \balpha \right] =
g(\bx_i)$.

The second sum in Eq.~\eqref{eq:objective-rewritten} can be seen as a regularizer enforcing all $h(T_\balpha
\mathbf{x}_i)$ to be close to their average value, i.e., the feature representation is sought to be approximately invariant to the
transformations $T_\balpha$. To show this we use the convexity of the
function $\log \norm{ \exp(\cdot)}_1$ and
Jensen's inequality, which yields (proof in Appendix~\ref{appendix:formal}):
\begin{equation}
 \me_\balpha \left[ \log \norm{\exp (h(T_\balpha \mathbf{x}_i))}_1 \right] - \log
 \norm{\exp (\mathbf{W} \avgg[i])}_1 \geq 0 .
 \label{eq:jensens}
\end{equation}
If the feature representation is perfectly invariant, then
$h(T_\balpha \bx_i) = \bW \avgg[i]$ and inequality~\eqref{eq:jensens}
turns to equality, meaning that the regularizer reaches its global
minimum.

\subsection{Conceptual Comparison to Previous Unsupervised Learning Methods}
\label{sect:properties}

Suppose we want to unsupervisedly learn a feature representation useful for a recognition task, for example classification. The mapping from input images $\bx$ to a feature
representation $g(\bx)$ should then satisfy two requirements: (1) there must be at least one feature
that is similar for images of the same category $\by$ (invariance); (2) there must be at least one feature
that is sufficiently different for images of different categories (ability to discriminate).

Most unsupervised feature learning methods aim to learn such a
representation by modeling the input distribution
$p(\mathbf{x})$. This is based on the assumption that a good model of
$p(\mathbf{x})$ contains information about the category distribution $p(\mathbf{y}|\mathbf{x})$.
That is, if a representation is learned, from which a
given sample can be reconstructed perfectly, then the representation is
expected to also encode information about the category of the sample (ability to discriminate).
Additionally, the learned representation should be invariant to variations in the
samples that are irrelevant for the classification
task, i.e., it should adhere to the manifold hypothesis (see
e.g. \citep{Rifai_NIPS2011} for a recent discussion).
Invariance is classically achieved by regularization of the
latent representation, e.g., by enforcing sparsity \citep{Kavukcuoglu_NIPS2010}
or robustness to noise \citep{Vincent_ICML2008}.


In contrast, the discriminative objective in Eq.~\eqref{eq:objective}
does not directly model the input distribution $p(\mathbf{x})$ but learns a
representation that discriminates between input samples. The representation
is not required to reconstruct the input, which is unnecessary in a
recognition or matching task. This leaves more degrees of freedom to model the desired
variability of a sample. As shown in our analysis
(see Eq. \eqref{eq:jensens}), we enforce invariance to transformations applied during surrogate data creation by requiring the
representation $g(T_{\balpha}\bx_i)$ of the transformed image patch to be predictive
of the surrogate label assigned to the original image patch $\bx_i$.

It should be noted that this approach assumes that the transformations $T_{\balpha}$ do not change the identity of the image content. For example, if we use a color transformation we will force the network to be invariant to this change and cannot expect the extracted features to perform well in a task relying on color information (such as
differentiating black panthers from pumas)\footnote{Such cases could be covered either by careful selection of applied
transformations or by combining features from multiple networks trained
with different sets of transformations and letting the final (supervised)
classifier choose which features to use.}.

\section{Experiments: Classification}
To compare our discriminative approach to previous unsupervised feature learning methods, we report classification results on the STL-10~\citep{Coates_AISTATS2010}, CIFAR-10~\citep{Krizhevsky_thesis2009},
Caltech-101~\citep{FeiFei_CVPR2004} and Caltech-256~\citep{Griffin_2007} datasets.

\subsection{Experimental Setup}
The datasets we tested on differ in the number of classes ($10$ for CIFAR and STL, $101$ for Caltech-101, $256$ for Caltech-256)
and the number of samples per class. STL is especially well suited
for unsupervised learning as it contains a large set of $100,\!000$
unlabeled samples. In all experiments, except for the dataset transfer experiment,
we extracted surrogate training data from the unlabeled subset of STL-10. When testing on CIFAR-10, we resized the images from $32 \times 32$ pixels
to $64 \times 64$ pixels to make the scale of depicted objects more similar to the other datasets. Caltech-101 images were resized to $150 \times 150$ pixels and Caltech-256 images to $256 \times 256$ pixels (Caltech-256 images have on average higher resolution than Caltech-101 images, so not downsampling them so much allows to preserve more fine details).

We worked with three network architectures. A smaller network was used to evaluate
the influence of different components of the augmentation procedure on
classification performance. It consists of two convolutional layers
with $64$ filters each, followed by a fully connected layer with $128$
units. This last layer is succeeded by a softmax
layer, which serves as the network output. This network will be referred to as 64c5-64c5-128f as explained in Appendix~\ref{appendix:network_architecture}.

To compare our method to the
state-of-the-art we trained two bigger networks: a network that consists of three convolutional layers with $64$,
$128$ and $256$ filters respectively followed by a fully connected
layer with $512$ units (64c5-128c5-256c5-512f), and an even larger network, consisting of three convolutional layers with $92$,
$256$ and $512$ filters respectively and a fully connected
layer with $1024$ units (92c5-256c5-512c5-1024f).

In all these models all convolutional filters are
connected to a $5 \times 5$ region of their input. $2 \times 2$
max-pooling was performed after the first and second convolutional
layers. Dropout~\cite{Hinton_arxiv2012, Srivastava_JMLR2014} was applied to the fully connected
layers. We trained the networks using an implementation based on \emph{Caffe}~\cite{caffe}.
Details on the training procedure and hyperparameter settings are provided in Appendix~\ref{appendix:training_details}.

\addtocounter{footnote}{-1}
\setlength{\tabcolsep}{2pt}
\begin{table*}
  \caption[]{Classification accuracies on several datasets
    (in percent).
  $*$ Average per-class accuracy\protect \footnotemark~ $78.0 \% \pm 0.4
    \%$.
    $\dagger$ Average per-class accuracy~ $85.0 \% \pm 0.7
    \%$. $\ddagger$ Average per-class accuracy $85.8 \% \pm 0.7
    \%$. } 
      \label{tbl:classification}
  \begin{minipage}{\textwidth}
  \vskip 0.1in
  \hskip -0.05in
  \centering
    \small{
      \begin{tabular}{l|c|c|c|c|c|c}
      \hline
                                                 Algorithm    &  STL-10                 &     CIFAR-10(400)       &     CIFAR-10        &      Caltech-101       & Caltech-256(30)& \#features \\ \hline
      Convolutional K-means Network~\cite{Coates_NIPS2011}    &  $60.1 \pm 1$           &  $70.7 \pm 0.7$         &  $82.0$             &        ---             & --- &8000            \\
      Multi-way local pooling~\cite{Boureau_ICCV2011}         &         ---             &          ---            &        ---          &      $77.3\pm0.6$      & $41.7$ &$1024 \times 64$ \\
      Slowness on videos~\cite{Zou_NIPS2012}                  &  $61.0$                 &       ---               &      ---            &      $74.6$            & --- &556 \\
      Hierarchical Matching Pursuit (HMP)~\cite{Bo_ISER2012}  &  $64.5 \pm 1$           &       ---               &      ---            &        ---             & --- &1000\\
      Multipath HMP~\cite{Bo_CVPR2013}                        &     ---                 &       ---               &      ---            &  $82.5 \pm 0.5$        & $50.7$ &5000 \\
      View-Invariant K-means~\cite{Hui_ICML2013}              &  $63.7$                 & $72.6 \pm 0.7$          &    $81.9$           &        ---             & --- &6400\\ \hline
      Exemplar-CNN (64c5-64c5-128f)                           &  $67.1 \pm 0.2$         & $69.7 \pm 0.3$          &    $76.5$           &  $79.8 \pm 0.5^*$      & $ 42.4 \pm 0.3 $&256 \\
      Exemplar-CNN (64c5-128c5-256c5-512f)                    & $72.8 \pm 0.4$          & $75.4 \pm 0.2$          & $82.2$              & $86.1 \pm 0.5^\dagger$ & $51.2 \pm 0.2$ &960 \\
      Exemplar-CNN (92c5-256c5-512c5-1024f)                   & $\mathbf{74.2 \pm 0.4}$ & $\mathbf{76.6 \pm 0.2}$ & $\mathbf{84.3}$     & $\mathbf{87.1 \pm 0.7}^\ddagger$ & $\mathbf{53.6 \pm 0.2}$ &1884 \\ \hhline{=======}
      Supervised state of the art                             &  $70.1$\cite{Swersky_NIPS2013} & ---      & $92.0$ \cite{Lee_NIPS2014} & $91.44$ \cite{He_ECCV2014} & $70.6$ \cite{Zeiler_ECCV2014}
       &---\\ \hline
      \end{tabular}
    }
  \end{minipage}
  \vskip -0.1in
\end{table*}
\setlength{\tabcolsep}{6pt}

At test time we applied a network to arbitrarily sized images by convolutionally computing the responses of all the network layers except the top softmax (that is, we computed the responses of convolutional layers normally and then slided the fully connected layers on top of these).
To the feature maps of each layer we applied the pooling
method that is commonly used for the respective dataset:
\begin{enumerate}
 \item[1)] 4-quadrant
max-pooling, resulting in $4$ values per feature map, which is the
standard procedure for STL-10 and CIFAR-10~\cite{Coates_NIPS2011,
  Zou_NIPS2012, Bo_ISER2012, Hui_ICML2013}
 \item[2)] 3-layer spatial
pyramid, i.e. max-pooling over the whole image as well as within 4
quadrants and within the cells of a $4 \times 4$ grid, resulting in
$1+4+16=21$ values per feature map, which is the standard for
Caltech-101 and Caltech-256~\cite{Boureau_ICCV2011, Zou_NIPS2012,
  Bo_CVPR2013}
\end{enumerate}
Finally, we trained a one-vs-all linear support vector machine
(SVM) on the pooled features.

On all datasets we used the standard training and test protocols. On STL-10 the SVM
was trained on 10 pre-defined folds of the training data. We report the mean and standard
deviation achieved on the fixed test set. For CIFAR-10 we
report two results:
\begin{enumerate}
\item[1)] Training the SVM on the whole CIFAR-10 training set
(called \textit{CIFAR-10})
\item[2)] The average over 10 random selections of 400 training
samples per class (called \textit{CIFAR-10(400)})
\end{enumerate}
For Caltech-101 we follow the usual
protocol of selecting 30 random samples per class for training and not
more than 50 samples per class for testing. For Caltech-256 we randomly selected 30 samples per class for training and used the rest for testing. Both for Caltech-101 and Caltech-256 we repeated the testing procedure 10 times.

\subsection{Classification Results\label{sec:classification}}
In Table~\ref{tbl:classification} we compare Exemplar-CNN to several
unsupervised feature learning methods, including the current state of the art on each dataset. We also list the state of the art for methods involving supervised feature learning (which is not directly comparable).
Additionally we show the dimensionality of the feature vectors produced by each
method before final pooling.
The smallest network was trained on $8000$ surrogate classes containing $150$
samples each and the larger ones on $16000$ classes with $100$ samples each.

The features extracted from both larger networks outperform the
best prior result on all datasets. This is despite the fact that the
dimensionality of the feature vectors is smaller than that of most
other approaches and that the networks are trained on the STL-10
unlabeled dataset (i.e. they are used in a transfer learning manner
when applied to CIFAR-10 and Caltech). The increase in performance is especially pronounced when only few
labeled samples are available for training the SVM, as is the case
for all the datasets except full CIFAR-10.
This is in agreement with previous evidence that with increasing feature vector dimensionality and number of labeled samples, training an SVM
becomes less dependent on the quality of the features \cite{Coates_NIPS2011,Hui_ICML2013}.
Remarkably, on STL-10 we achieve an accuracy of $74.2 \%$, which is a
large improvement over all previously reported results.

\footnotetext{
  On Caltech-101 one can either measure average accuracy over all
  samples (average overall accuracy) or calculate the accuracy for
  each class and then average these values (average
  per-class accuracy). These differ, as some classes contain fewer than
  $50$ test samples. Most researchers in ML use average overall
  accuracy.
  }
\subsection{Detailed Analysis}

We performed additional experiments using the 64c5-64c5-128f network to study the effect of various design choices in Exemplar-CNN training and validate the invariance properties of the learned features.


\begin{table*}[t!]
  \caption[]{Classification accuracies with clustering
    (in percent).}
      \label{tbl:cluster}
  \begin{minipage}{\textwidth}
  \centering
    \small{
      \begin{tabular}{l|c|c|c|c|c}
      \hline
                                  Algorithm    &  STL-10                  &     CIFAR-10(400)        &     CIFAR-10     &      Caltech-101        &  Caltech-256(30)    \\ \hline
      64c5-64c5-128f                           & $69.5 \pm 0.4$           & $70.8 \pm 0.2$           &    $76.8$        & $79.5 \pm 0.6$          & $42.9 \pm 0.3$      \\
      64c5-128c5-256c5-512f                    & $74.9 \pm 0.4$           & $75.7 \pm 0.2$           & $82.6$           & $85.7 \pm 0.6$          & $51.4  \pm 0.4$      \\
      92c5-256c5-512c5-1024f                   & $\mathbf{75.4 \pm 0.3}$  & $\mathbf{77.4 \pm 0.2}$  & $\mathbf{84.3}$  & $\mathbf{87.2 \pm 0.6}$ & $\mathbf{53.7 \pm 0.6}$      \\
      \end{tabular}
    }
  \end{minipage}
  \vskip -0.1in
\end{table*}

\subsubsection{Number of Surrogate Classes}
We varied the number $N$ of surrogate classes between $50$ and $32000$. As a sanity check, we also tried classification with random
filters. The results are shown in Fig.~\ref{fig:var_num_classes}.

Clearly, the classification accuracy increases with the number of surrogate classes
until it reaches an optimum at about $8000$ surrogate classes after which it did not change or even
decreased. This is to be expected: the larger the number of surrogate classes, the more likely it is to draw very similar or even identical samples, which are hard or impossible to discriminate.
Few such cases are not detrimental to the classification performance, but as soon as such collisions dominate the set of surrogate labels, the discriminative loss is no longer reasonable and training
the network to the surrogate task no longer succeeds. To
check the validity of this explanation we also plot in
Fig.~\ref{fig:var_num_classes} the validation error on the surrogate data
after training the network. It rapidly grows as the number of surrogate
classes increases, showing that the surrogate classification task gets harder
with a growing number of classes. We observed that larger, more powerful networks
reach their peak performance for more surrogate classes than smaller networks. However,
the performance that can be achieved with larger networks saturates (not shown in
the figure).

It can be seen as a limitation that sampling too many, too similar images for training can even decrease the performance of the learned features. It makes the number and selection of samples a relevant parameter of the training procedure. However, this drawback can be avoided for example by clustering.

To demonstrate this, given the STL-10 unlabeled dataset containing 100,000 images, we first train a 64c5-128c5-256c5-512f Exemplar-CNN on a subset of 16,000 image patches. We then use this Exemplar-CNN to extract descriptors of all images from the dataset and perform clustering similar to~\cite{Singh_ECCV2012}. After discarding noisy and very similar clusters automatically (see Appendix~\ref{app:clustering} for details), this leaves us with $6510$ clusters with approximately $10$ images in each of them. To the images in each cluster we then apply the same augmentation as in the original Exemplar-CNN. Each augmented cluster serves as a surrogate class for training. Table~\ref{tbl:cluster} shows the classification performance of the features learned by CNNs from this training data. Clustering increases the classification accuracy on all datasets, in particular on STL by up to $2.4$\%, depending on the network. This shows that the small modification allows the approach to make use of large amounts of data. Potentially, using even more data or performing clustering and network training within a unified framework could further improve the quality of the learned features.

\subsubsection{Number of Samples per Surrogate Class}


\begin{figure}
\centering
  \includegraphics[width=.45\textwidth]{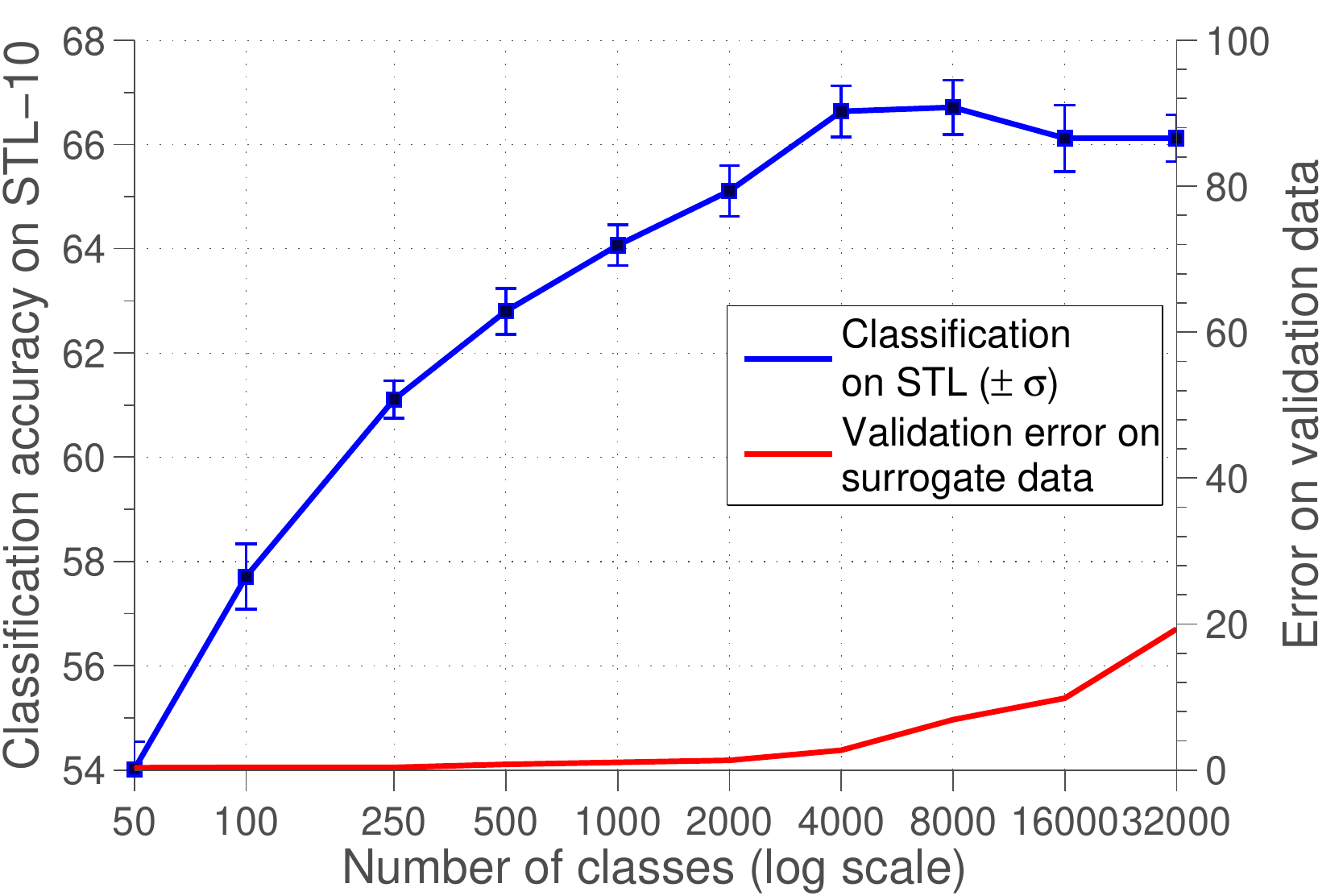}
  \caption{Influence of the
    number of surrogate training classes. The validation error on
    the surrogate data is shown in red.
    Note the different y-axes for the two curves.}
  \label{fig:var_num_classes}
\end{figure}

\begin{figure}
  \centering
  \includegraphics[width=.45\textwidth]{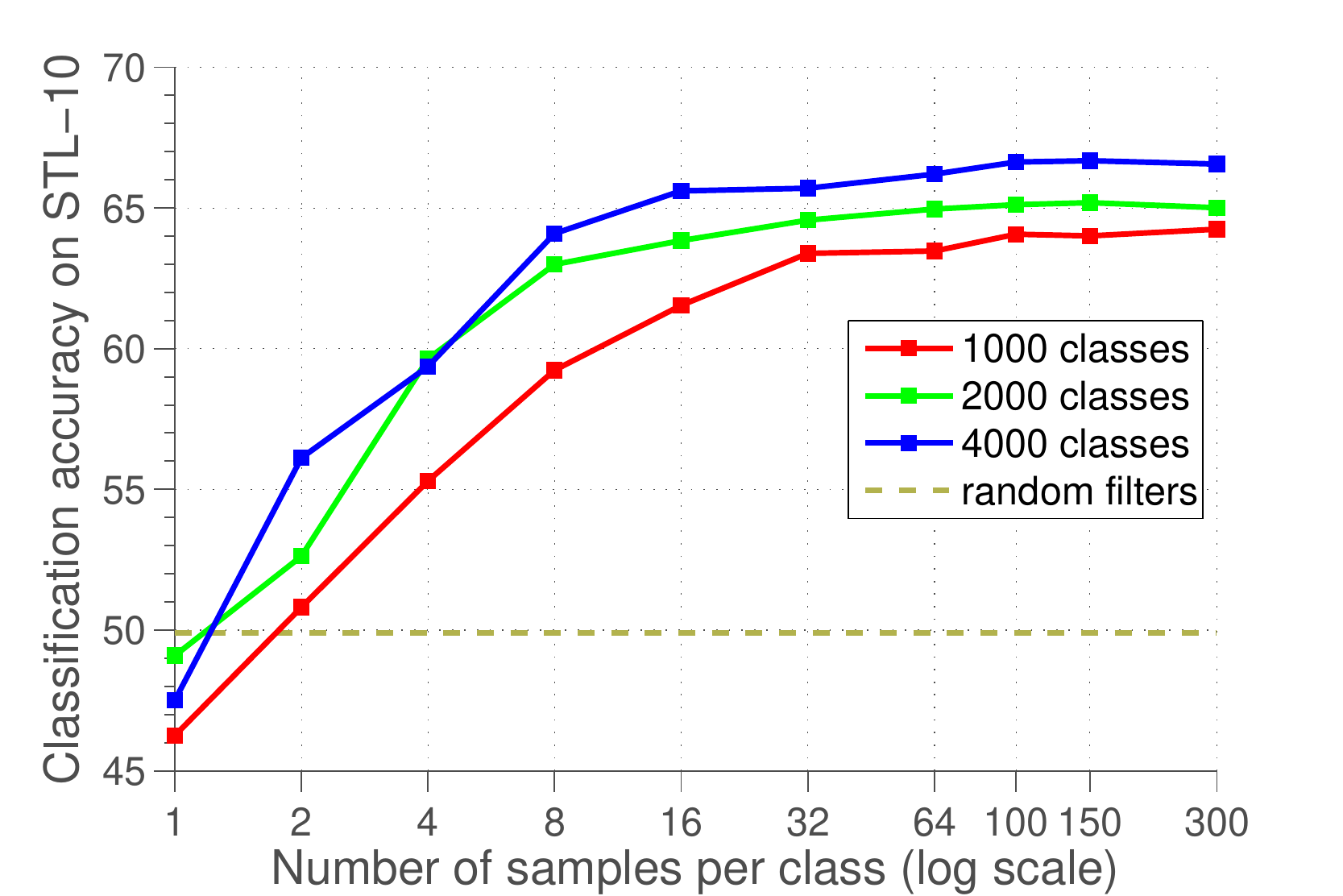}
  \caption{Classification performance on STL for different numbers
    of samples per class. Random filters can be seen as '0 samples per
    class'.}
  \label{fig:var_num_samples}
\end{figure}

Fig.~\ref{fig:var_num_samples} shows the classification accuracy when
the number $K$ of training samples per surrogate class varies between $1$ and
$300$. The performance improves with more samples per surrogate class and saturates at
around $100$ samples. This indicates that this amount is sufficient to approximate the formal
objective from Eq.~\eqref{eq:formal_objective}, hence further increasing the number of samples does not significantly change the optimization problem. On the other hand, if the number of samples is too small, there is not enough data to learn the desired invariance properties.

\subsubsection{Types of Transformations}

We varied the transformations used for creating the surrogate data to analyze their influence on the final
classification performance. The set of 'seed' patches was fixed. The result is shown in
Fig.~\ref{fig:removing_transformations}.
The value '$0$' corresponds to applying random compositions of all elementary transformations: scaling, rotation, translation, color variation, and contrast variation. Different columns of the plot show the difference
in classification accuracy as we discarded some types of elementary transformations.

Several tendencies can be observed. First, rotation and scaling have only a minor impact on the performance, while translations, color variations and contrast variations are significantly more important.
Secondly, the results on STL-10 and CIFAR-10 consistently show that
spatial invariance and color-contrast invariance are approximately of
equal importance for the classification performance. This indicates that
variations in color and contrast, though often neglected,
may also improve performance in a supervised learning scenario.
Thirdly, on Caltech-101 color and contrast transformations are much more important compared to spatial transformations than on the two other datasets. This is not surprising, since Caltech-101 images are often well aligned, and this dataset
bias makes spatial invariance less useful.

We tried applying several other transformations (occlusion,
affine transformation, additive Gaussian noise) in addition to the ones shown in
Fig.~\ref{fig:removing_transformations}, none of which seemed to improve the
classification accuracy.
For the matching task in Section~\ref{sect:matching}, though, we found that using blur as an
additional transformation improves the performance.



\begin{figure}
  \begin{center}
  \hspace*{-10pt}
    \includegraphics[width=.52\textwidth]{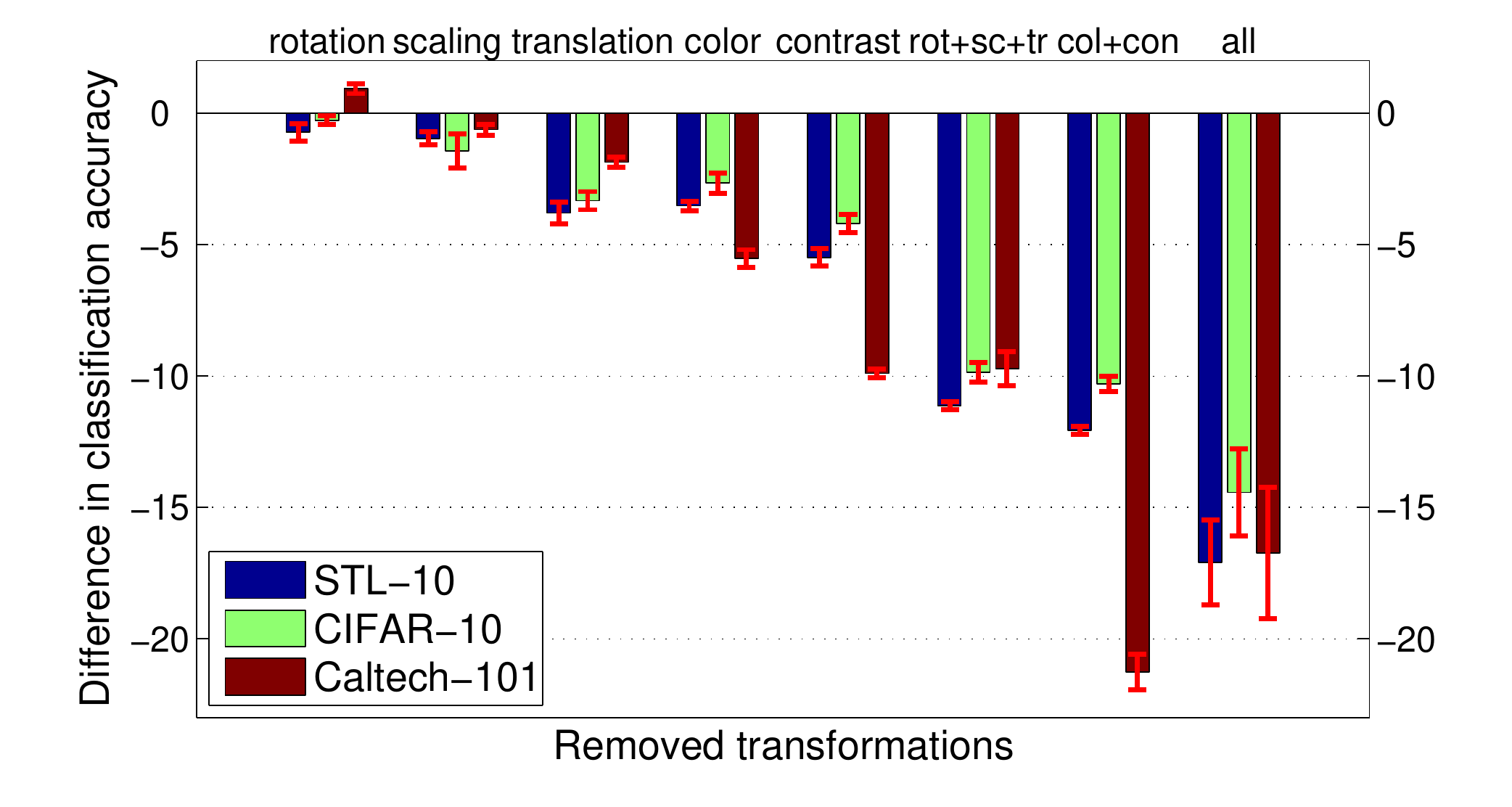}
    \vspace*{-20pt}
    \caption{Influence of
    removing groups of transformations during generation of the
    surrogate training data. Baseline ('$0$' value) is applying all
    transformations. Each group of three bars corresponds to removing
    some of the transformations.}
    \label{fig:removing_transformations}
  \end{center}
  \vspace*{-15pt}
\end{figure}

\begin{figure*}
\begin{center}
\includegraphics[width=.6\textwidth]{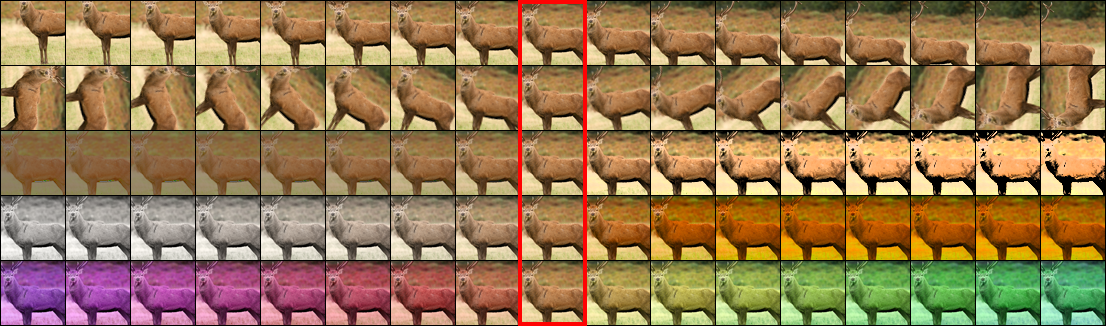}
\end{center}
\vspace*{10pt}
\begin{tabular}{ccc}
  \hspace*{-10 pt}\includegraphics[width=.35\textwidth]{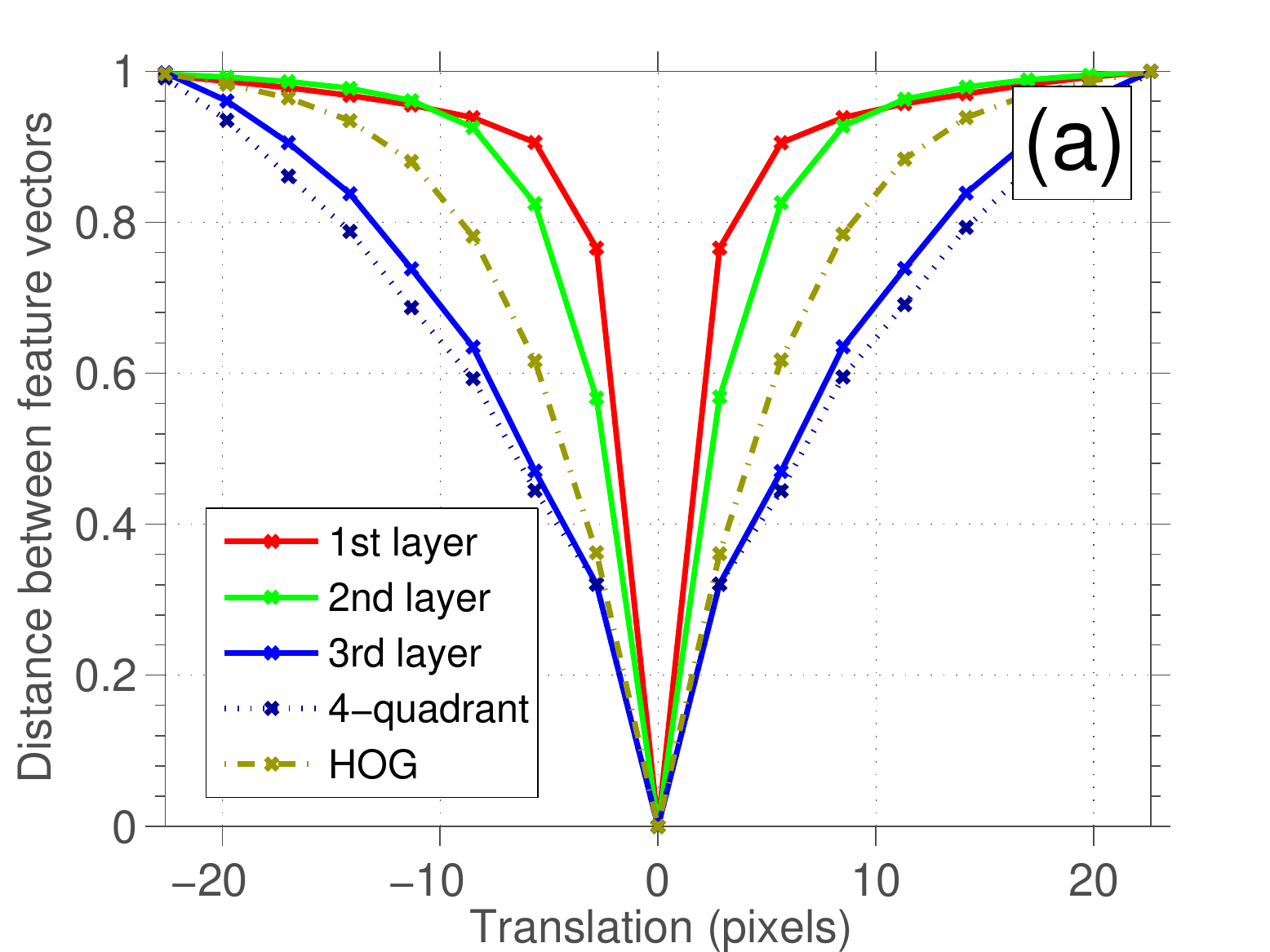} &
  \hspace*{-20 pt}\includegraphics[width=.35\textwidth]{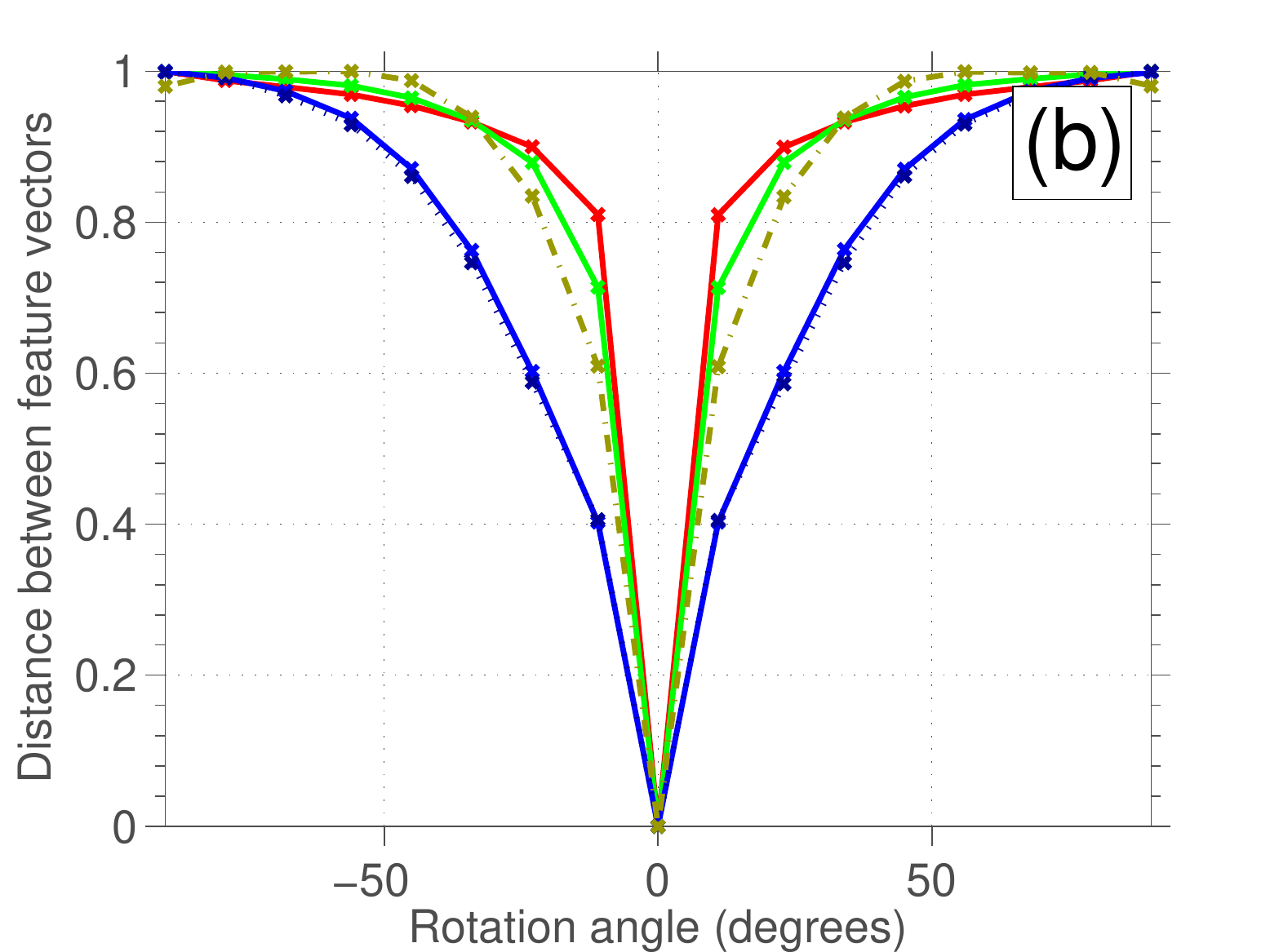} &
  \hspace*{-20 pt}\includegraphics[width=.35\textwidth]{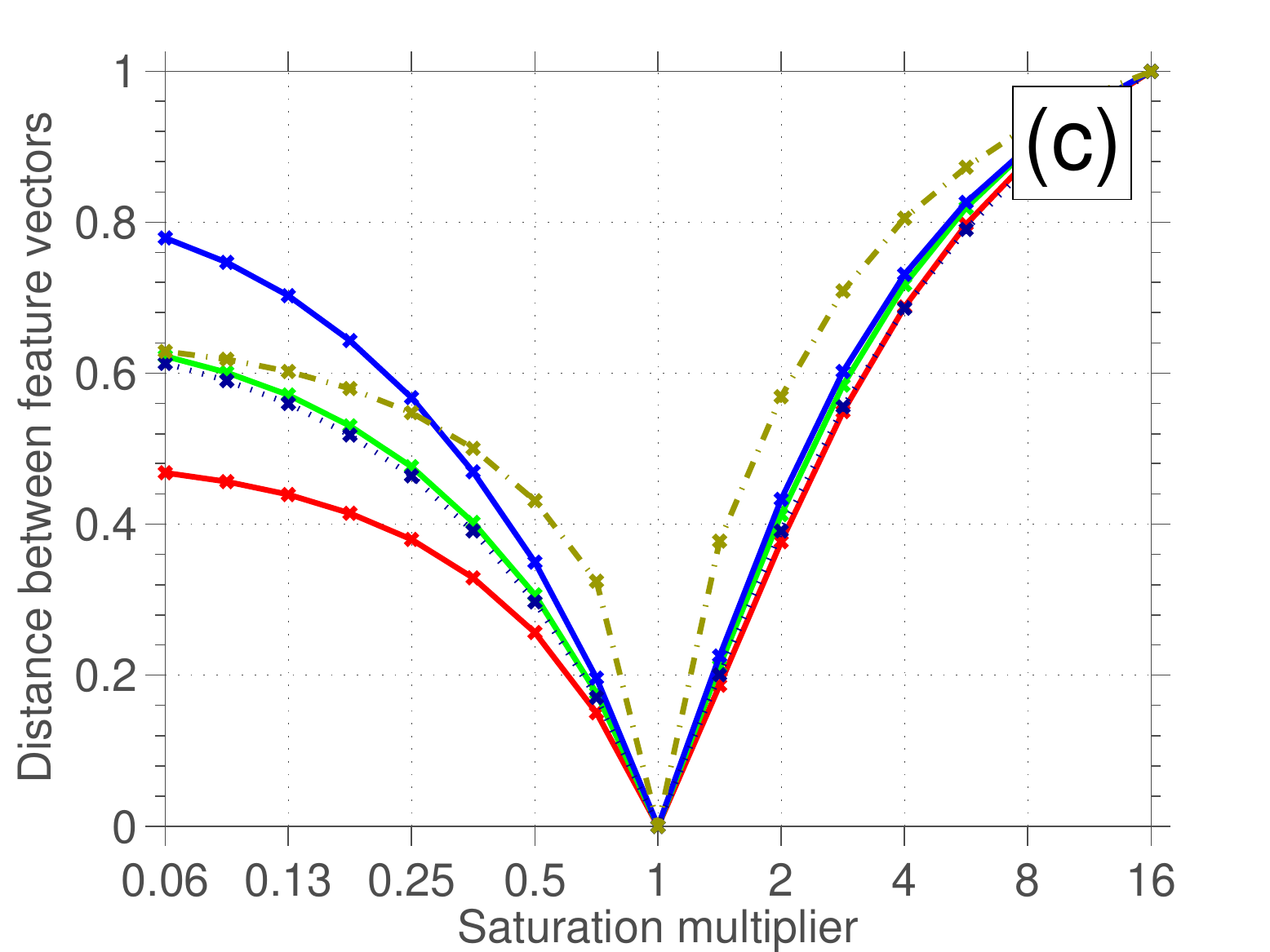} \\
  \hspace*{-10 pt}\includegraphics[width=.35\textwidth]{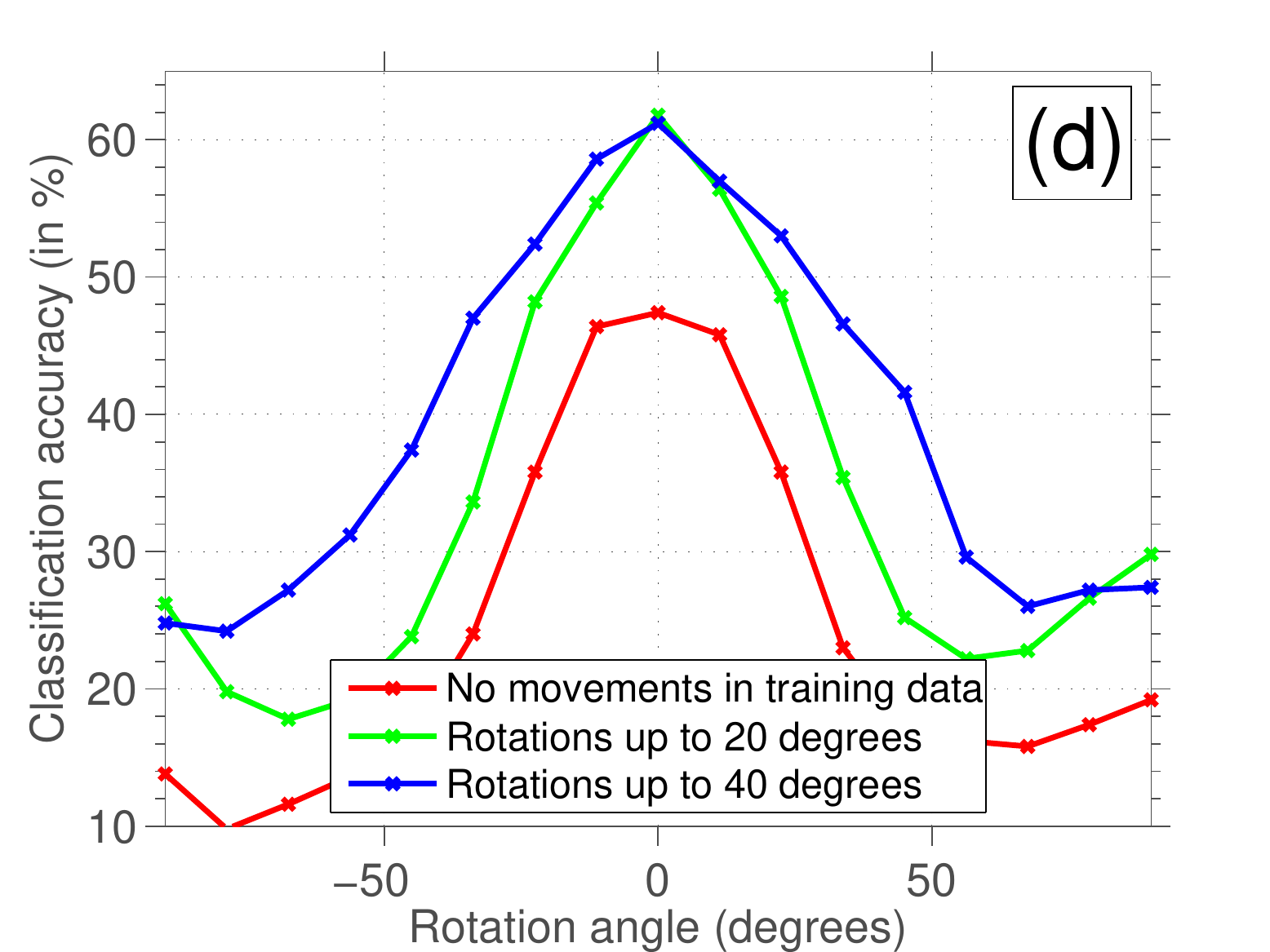} &
  \hspace*{-20 pt}\includegraphics[width=.35\textwidth]{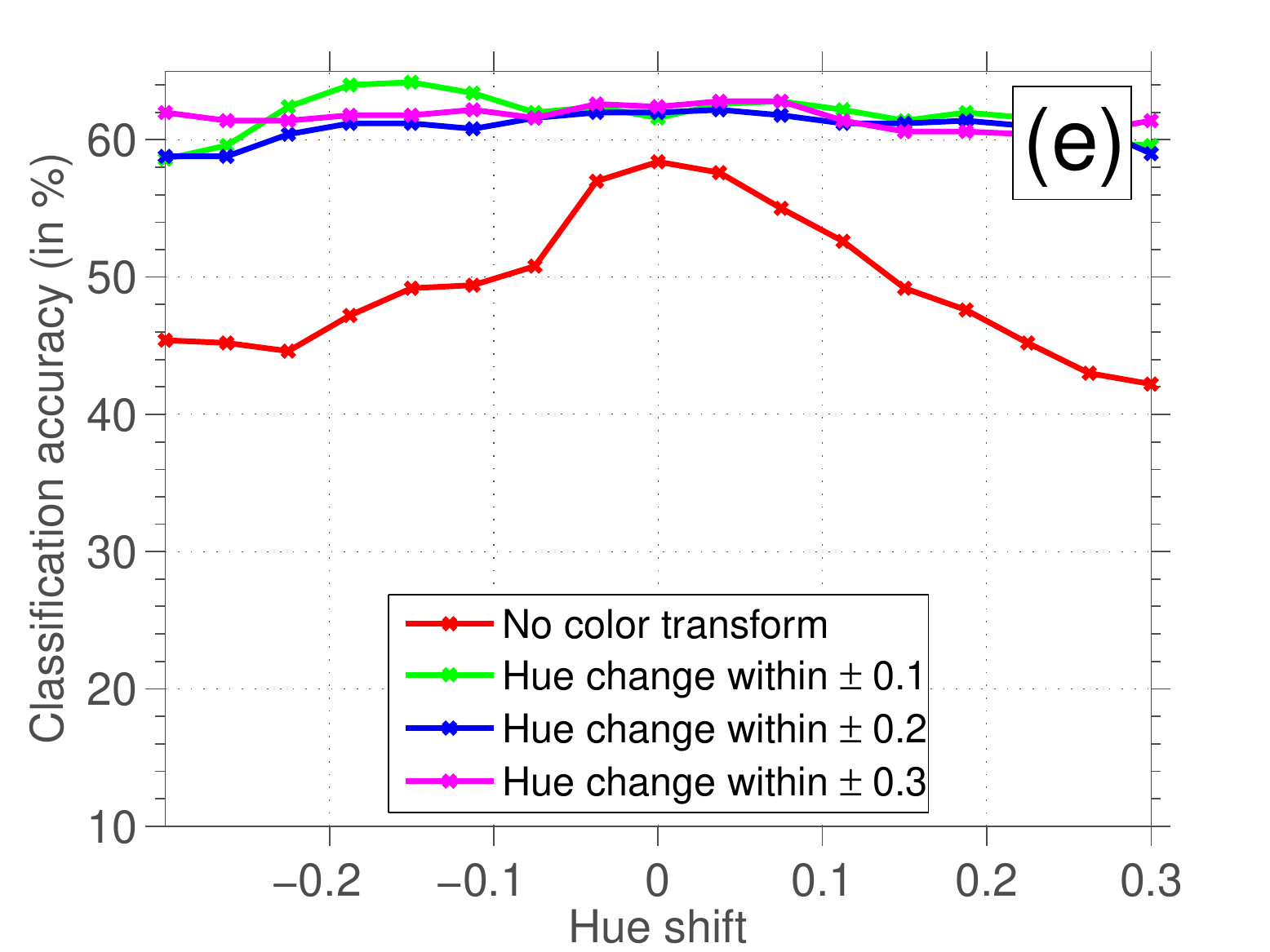} &
  \hspace*{-20 pt}\includegraphics[width=.35\textwidth]{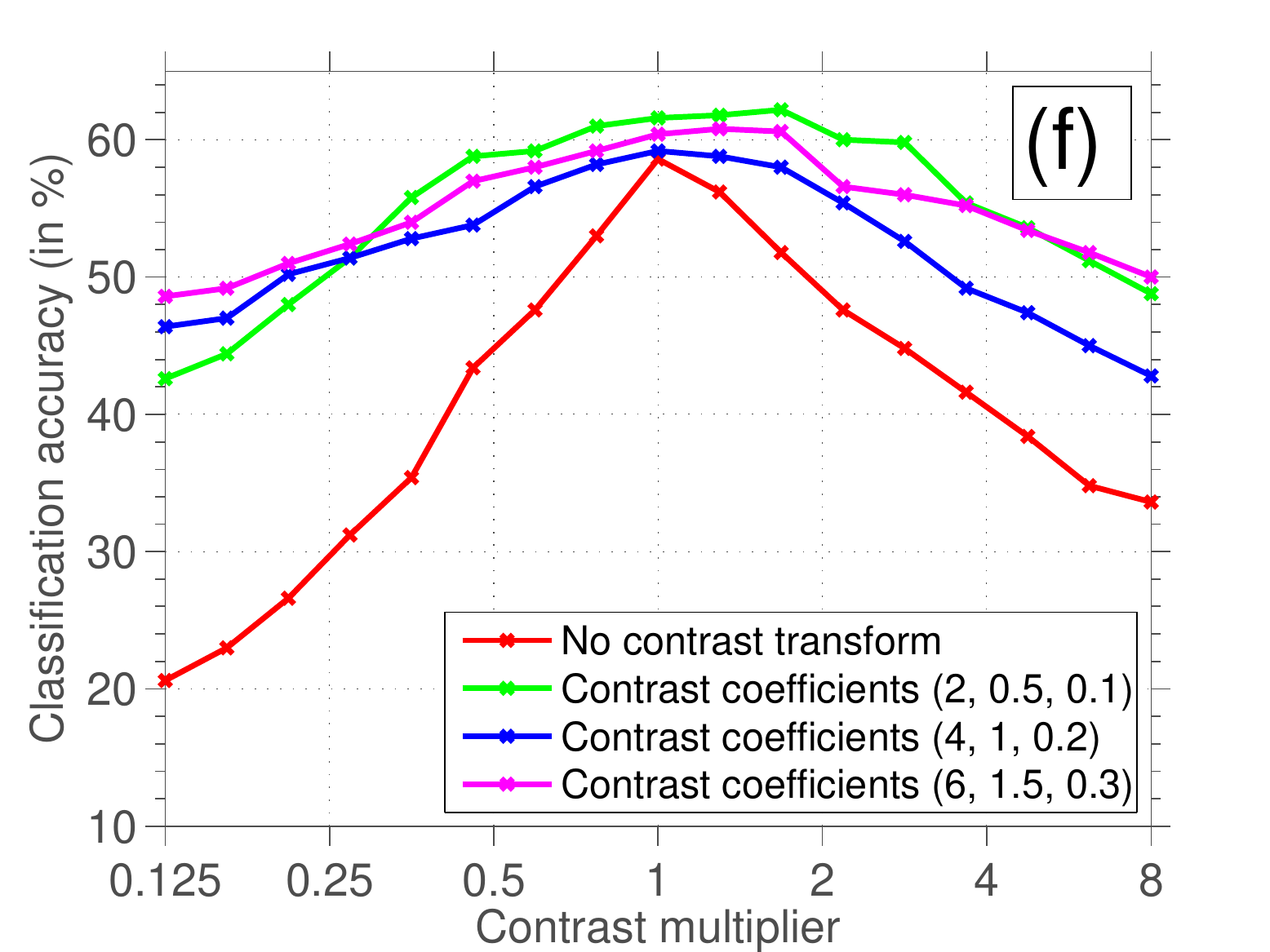}
\end{tabular}
\vskip -0.1in
\caption{Invariance properties of the feature representation learned by Exemplar-CNN. Top: transformations applied to an image patch (translation, rotation, contrast, saturation, color). Bottom: invariance of different feature representations. (a)-(c): Normalized Euclidean distance between feature vectors of the original and the translated image patches vs. the magnitude of the transformation, (d)-(f): classification performance on transformed image patches vs. the magnitude of the transformation for various magnitudes of transformations applied for creating the surrogate training data.}
\label{fig:invariance}
\vskip -0.1in
\end{figure*}

\hspace*{20pt}
\begin{figure}
\centering
\begin{tabular}{c}
  \includegraphics[width=.8\columnwidth]{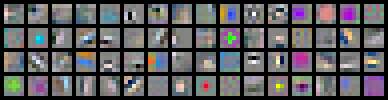} \\
  \includegraphics[width=.8\columnwidth]{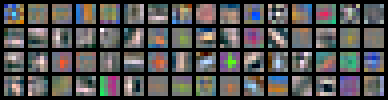} \\
  \includegraphics[width=.8\columnwidth]{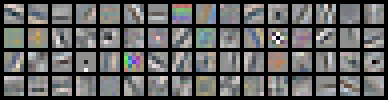}
\end{tabular}
\caption{Filters learned by first layers of 64c5-64c5-128f networks when training on surrogate data from various dataset. Top~-- from STL-10, middle~-- CIFAR-10, bottom~-- Caltech-101. }
\label{fig:layer1_filters}
\end{figure}

\subsubsection{Influence of the Dataset\label{sec:dataset}}
We applied our feature learning algorithm to images sampled from three
datasets~-- STL-10 unlabeled dataset, CIFAR-10 and Caltech-101~-- and evaluated the performance of the learned feature representations on classification tasks on these datasets. We used the 64c5-64c5-128f network for this experiment.


We show the first layer filters learned from the three datasets in
Fig.~\ref{fig:layer1_filters}. Note how filters qualitatively differ depending on the dataset they were trained on.

Classification results are shown in Table~\ref{tbl:datasets_comparison}. The best
classification results for each dataset are obtained when training on
the patches extracted from the dataset itself. However, the difference
is not drastic, indicating that the learned features generalize well to other datasets.

\setlength{\tabcolsep}{6pt}
\begin{table}
    \caption{Dependence of classification performance (in \%) on the
      training and testing datasets. Each column corresponds to
      different test data, each row to different training data
      (i.e. source of seed patches). We used the 64c5-64c5-128f network for this experiment.}
    \label{tbl:datasets_comparison}
    \renewcommand{\arraystretch}{1.1}
  \begin{sc}
  \begin{center}
    \begin{tabular}{l|c|c|c|}
    \cline{2-4}
                                             & \multicolumn{3}{|c|}{Testing}                                                  \\ \hline
    \multicolumn{1}{|c|}{Training}           & STL-10                    & CIFAR-10(400)                & Caltech-101              \\ \hline
    \multicolumn{1}{|c|}{STL-10}            & $\mathbf{67.1 \pm 0.3}$   & $69.7 \pm 0.3$          & $79.8 \pm 0.5$           \\
    \multicolumn{1}{|c|}{CIFAR-10}          & $64.5 \pm 0.4$            & $\mathbf{70.3 \pm 0.4}$ & $77.8 \pm 0.6$           \\
    \multicolumn{1}{|c|}{Caltech-101}        & $66.2 \pm 0.4$            & $69.5 \pm 0.2$          & $\mathbf{80.0 \pm 0.5}$  \\ \hline
    \end{tabular}
  \end{center}
  \end{sc}
\end{table}

\subsubsection{Influence of the Network Architecture on Classification Performance}

We perform an additional experiment to evaluate the influence of the
network architecture on classification performance. The results of
this experiment are shown in Table~\ref{tbl:architecture}. All
networks were trained using a surrogate training set containing either
$8000$ classes with $150$ samples each or $16000$ classes with $100$ samples each (for larger networks). We vary the number of layers,
layer sizes and filter sizes. Classification accuracy
generally improves with the network size indicating that our
classification problem scales well to relatively large networks without overfitting.

\setlength{\tabcolsep}{4pt}
\begin{table*}
      \caption{Classification accuracy depending on the network
        architecture. The name coding is as follows:  NcF stands for a convolutional layer with $N$ filters
        of size $F \times F$ pixels, Nf stands for a fully connected layer with
        $N$ units. For example, 64c5-64c5-128f denotes a network with
        two convolutional layers containing 64 filters spanning $5 \times 5$
        pixels each, followed by a fully connected layer with $128$
        units. We also show the number of surrogate classes used for training each network.}
      \label{tbl:architecture}
      \vspace*{10 pt}
\renewcommand{\arraystretch}{1.1}
  \centering
      \begin{tabular}{|l|c|c|c|c|c|}
      \hline
      Architecture                 &\#classes & STL-10          &  CIFAR-10(400)   &  CIFAR-10      &  Caltech-101            \\ \hline
      32c5-32c5-64f                &   8000   & $63.8 \pm 0.4$  &  $66.1 \pm 0.4$  &    $71.3$      &  $78.2 \pm 0.6$         \\ \hline
      64c5-64c5-128f               &   8000   & $67.1 \pm 0.3$  &  $69.7 \pm 0.3$  &    $75.7$      &  $79.8 \pm 0.5$         \\ \hline
      64c7-64c5-128f               &   8000   & $66.3 \pm 0.4$  &  $69.5 \pm 0.3$  &    $75.0$      &  $79.4 \pm 0.7$         \\ \hline
      64c5-64c5-64c5-128f          &   8000   & $68.5 \pm 0.3$  &  $70.9 \pm 0.3$  &    $77.0$      &  $82.2 \pm 0.7$         \\ \hline
      64c5-64c5-64c5-64c5-128f     &   8000   & $64.7 \pm 0.5$  &  $67.5 \pm 0.3$  &    $75.2$      &  $75.7 \pm 0.4$         \\ \hline
      128c5-64c5-128f              &   8000   & $67.2 \pm 0.4$  &  $69.9 \pm 0.2$  &    $76.1$      &  $80.1 \pm 0.5$         \\ \hline
      64c5-256c5-128f              &   8000   & $69.2 \pm 0.3$  &  $71.7 \pm 0.3$  &    $77.9$      &  $81.6 \pm 0.5$         \\ \hline
      64c5-64c5-512f               &   8000   & $69.0 \pm 0.4$  &  $71.7 \pm 0.2$  &    $79.3$      &  $82.9 \pm 0.4$         \\ \hline
      128c5-256c5-512f             &   8000   & $71.2 \pm 0.3$  &  $73.9 \pm 0.3$  &    $81.5$      &  $84.3 \pm 0.6$         \\ \hline
      128c5-256c5-512f             &  16000   & $71.9 \pm 0.3$  &  $74.3 \pm 0.3$  &    $81.4$      &  $84.6 \pm 0.6$         \\ \hline
      64c5-128c5-256c5-512f        &  16000   & $72.8 \pm 0.4$  &  $75.3 \pm 0.3$  &    $82.0$      &  $85.5 \pm 0.4$         \\ \hline
      92c5-256c5-512c5-1024f       &  16000   & $73.9 \pm 0.4$  &  $76.0 \pm 0.2$  &    $83.6$      &  $86.9 \pm 0.6$         \\ \hline
      \end{tabular}
\end{table*}

\subsubsection{Invariance Properties of the Learned Representation}
We analyzed to which extent the representation
learned by the network is invariant to the
transformations applied during training.
We randomly sampled $500$ images from the STL-10 test set and applied
a range of transformations (translation, rotation, contrast, color) to
each image. To avoid empty regions beyond the image boundaries when applying spatial transformations, we cropped the
central $64\times64$ pixel sub-patch from each $96\times 96$ pixel
image. We then applied two measures of invariance to these
patches.

First, as an explicit measure of invariance, we calculated the normalized Euclidean distance between
normalized feature vectors of the original image patch and the transformed one~\cite{Zou_NIPS2012} (see
Appendix~\ref{appendix:invariance_details} for details). The
downside of this approach is that the distance between extracted
features does not take into account how informative and discriminative they are.
We therefore evaluated a second measure ~-- classification
performance depending on the magnitude of the transformation applied
to the classified patches ~-- which does not come with this
problem. To compute the classification accuracy, we trained an SVM
on the central $64\times64$ pixel patches from one fold of the STL-10
training set and measured classification performance on all transformed versions of $500$
samples from the test set.

The results of both
experiments are shown in Fig.~\ref{fig:invariance}. Overall the experiment empirically confirms that the Exemplar-CNN objective leads
to learning invariant features. Features in the third layer and the final pooled feature representation compare favorably
to a HOG baseline (Fig. \ref{fig:invariance} (a), (b)). This is consistent with the results we get in Section~\ref{sect:matching} for descriptor matching, where we compare the features to SIFT (which is similar to HOG).

Fig.~\ref{fig:invariance}(d)-(f) further show that stronger transformations
in the surrogate training data lead to a more invariant classification
with respect to these
transformations. However, adding too much contrast
variation may deteriorate classification performance (Fig. \ref{fig:invariance} (f)).
One possible reason is that the contrast level can be a useful feature: for example, strong edges in an image are usually more important than weak ones.

\section{Experiments: Descriptor Matching}\label{sect:matching}

In recognition tasks, such as image classification and object detection, the invariance requirements are largely defined by object class labels. Consequently, providing these class labels already when learning the features should be advantageous. This can be seen in the comparison to the supervised state-of-the-art in Table~\ref{tbl:classification}, where supervised feature learning performs better than the presented approach.

In contrast, matching of interest points in two images should be independent of object class labels. As a consequence, there is no apparent reason, why feature learning using class annotation should outperform unsupervised feature learning. One could even imagine that the class annotation is confusing and yields inferior features for matching.

\subsection{Compared Features}

We compare the features learned by supervised and unsupervised convolutional networks and SIFT~\cite{Lo04} features. For a long time SIFT has been the preferred descriptor in matching tasks  (see~\cite{mikopami05} for a comparison).

As supervised CNN we used the AlexNet model trained on ImageNet available at~\cite{caffe}. The architecture of the network follows Krizhevsky~et~al.~\cite{Krizhevsky_NIPS2012} and contains 5~convolutional layers followed by 2~fully connected layers. In the experiments, we extract features from one of the 5 convolutional layers of the network. For large input patch sizes, the output dimensionality is high, especially for lower layers. For the descriptors to be more comparable to SIFT, we decided to max-pool the extracted feature map down to a fixed $4 \times 4$ spatial size which corresponds to the spatial resolution of SIFT pooling. Even though the spatial size is the same, the number of features per cell is larger than for SIFT.



As unsupervised CNN we evaluated the matching performance of the 64c5-128c5-256c5-512f architecture, referred to as Exemplar-CNN-orig in the following.
As the experiments show, neural networks cannot handle blur very well. Increasing image blur always leads to a matching performance drop. Hence we also trained another Exemplar-CNN to deal with this specific problem. First, we increased the filter size and introduced a stride of 2 in the first convolutional layer, resulting in the following architecture: 64c7s2-128c5-256c5-512f. This allows the network to  identify edges in very blurry images more easily. Secondly, we used unlabeled images from Flickr for training, because these represent the general distribution of natural images better than STL. Thirdly, we applied blur of variable strength to the training data as an additional augmentation. We thus call this network Exemplar-CNN-blur. As with AlexNet, we max-pooled the feature maps produced by the Exemplar-CNNs to a $4 \times 4$ spatial size.


\subsection{Datasets}
The common matching dataset by Mikolajczyk et al.~\cite{mikoijcv05} contains only $40$ image pairs. This dataset size limits the reliability of conclusions drawn from the results, especially as we compare various design choices, such as the depth of the network layer from which we draw the features.  We set up an additional dataset that contains $384$ image pairs. It was generated by applying 6 different types of transformations with varying strengths to $16$ base images we obtained from Flickr. These images were not contained in the set we used to train the unsupervised CNN.

To each base image we applied the geometric transformations \emph{rotation}, \emph{zoom}, \emph{perspective}, and \emph{nonlinear deformation}. These cover rigid and affine transformations as well as more complex ones. Furthermore we applied changes to \emph{lighting} and focus by adding \emph{blur}.
Each transformation was applied in various magnitudes such that its effect on the performance could be analyzed in depth. For each of the 16 base images we matched all the transformed versions of the image to the original one, which resulted in $384$
matching pairs.

The dataset from Mikolajczyk et al.~\cite{mikoijcv05} was not generated synthetically but contains real photos taken from different viewpoints or with different camera settings. While this reflects reality better than a synthetic dataset, it also comes with a drawback: the transformations are directly coupled with the respective images. Hence, attributing performance changes to either different image contents or to the applied transformations becomes impossible. In contrast, the new dataset enables us to evaluate the effect of each type of transformation independently of the image content.

\begin{figure*}
\begin{tabular}{cccc}
  \hspace*{-7 pt}\includegraphics[width=.25\textwidth]{sift_patchsize.pdf} &
  \hspace*{-7 pt}\includegraphics[width=.25\textwidth]{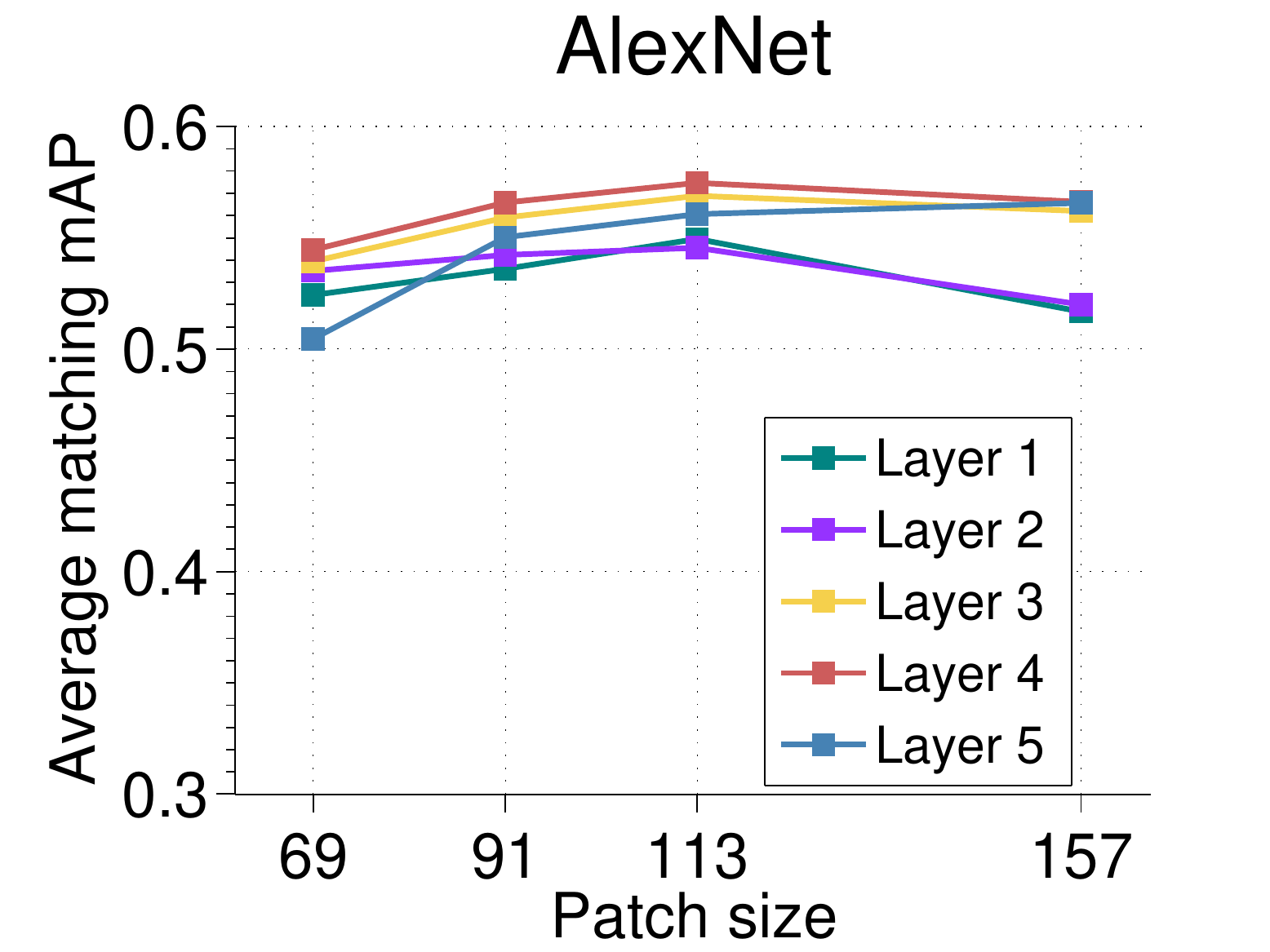} &
  \hspace*{-7 pt}\includegraphics[width=.25\textwidth]{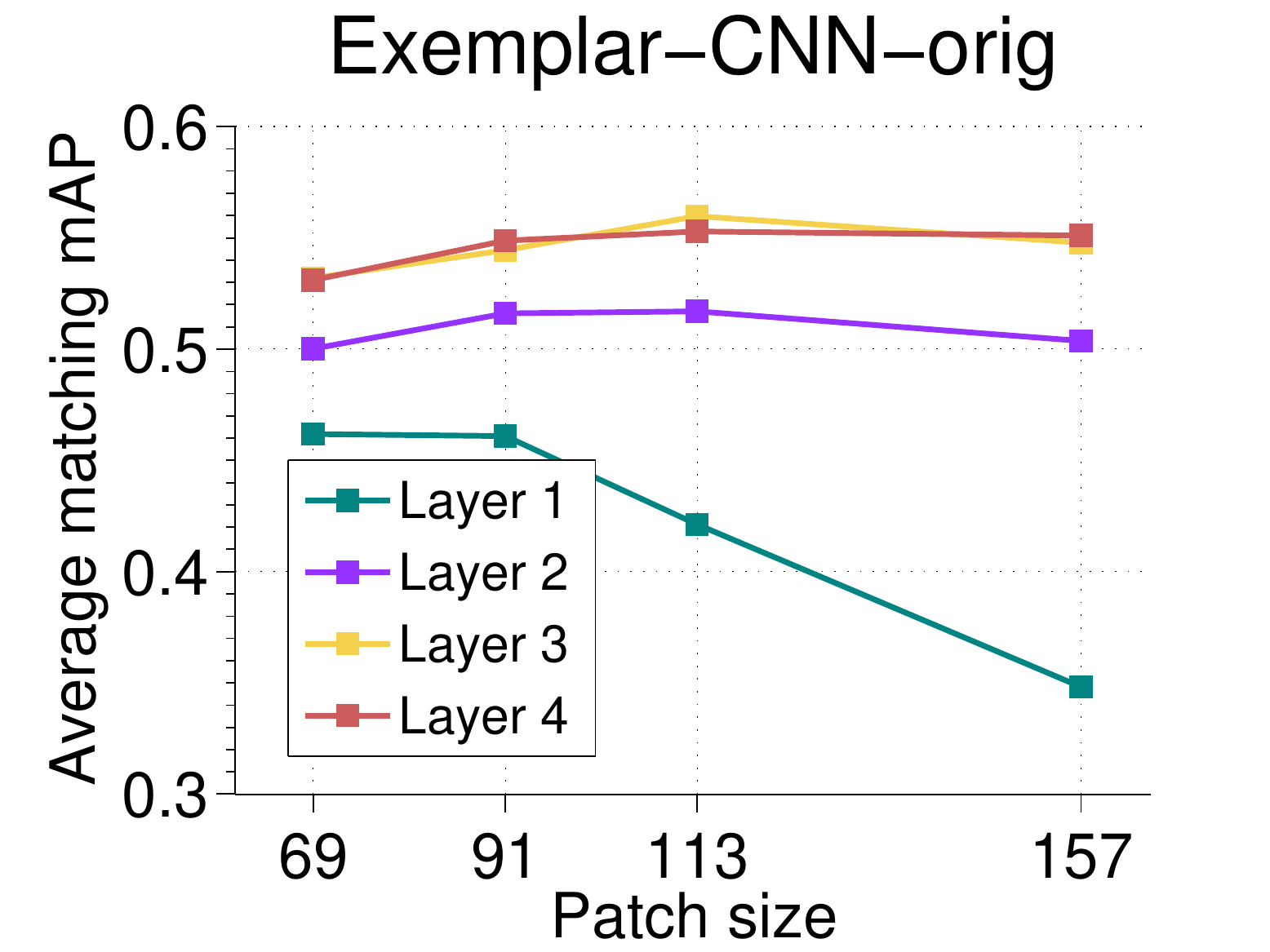} &
  \hspace*{-7 pt}\includegraphics[width=.25\textwidth]{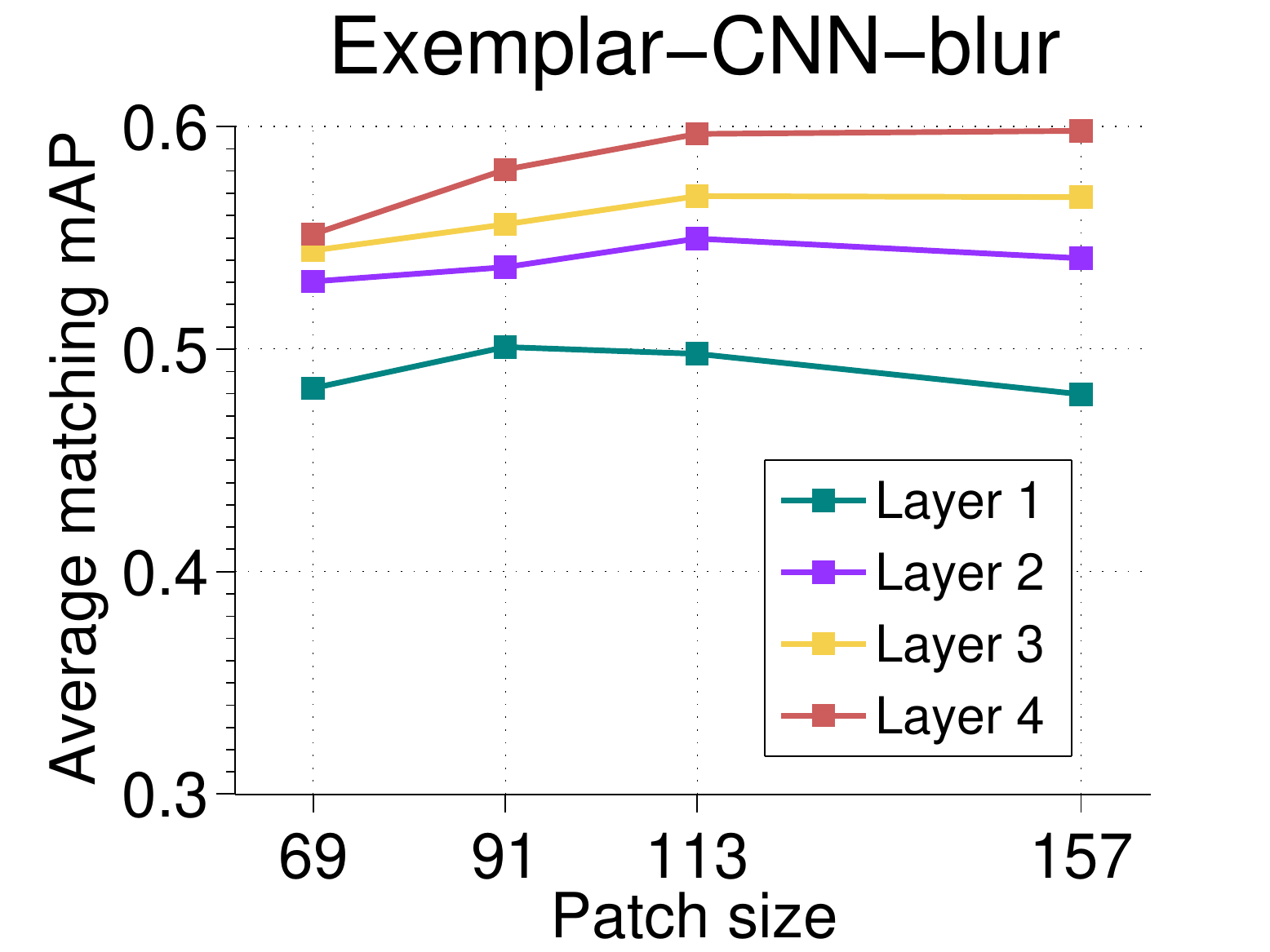} \\
\end{tabular}
\caption{Analysis of the matching performance depending on the patch size and the network layer at which features are computed.}
\label{fig:matching_patchsize}
\end{figure*}

\subsection{Performance Measure}
To evaluate the matching performance for a pair of images, we followed the procedure described in \cite{mikopami05}.
We first extracted elliptic regions of interest and corresponding image patches from both images using the \emph{maximally stable extremal regions} (MSER) detector~\cite{mser}.
We chose this detector because it was shown to perform consistently well in \cite{mikoijcv05} and it is widely used.
For each detected region we extracted a patch according to the region scale and rotated it according to its dominant orientation.
The descriptors of all extracted patches were greedily matched based on the Euclidean distance. This yielded a ranking of descriptor pairs.
A pair was considered as a true positive if the ellipse of the descriptor in the target image and the ground truth ellipse in the target image had an intersection over union (IOU) of at least $0.5$. All other pairs were considered false positives.
Assuming that a recall of 1 corresponds to the best achievable overall matching given the detections, we computed a precision-recall curve. The average precision, i.e., the area under this curve, was used as performance measure.

\begin{figure*}
\begin{tabular}{ccc}
  \includegraphics[width=.33\textwidth]{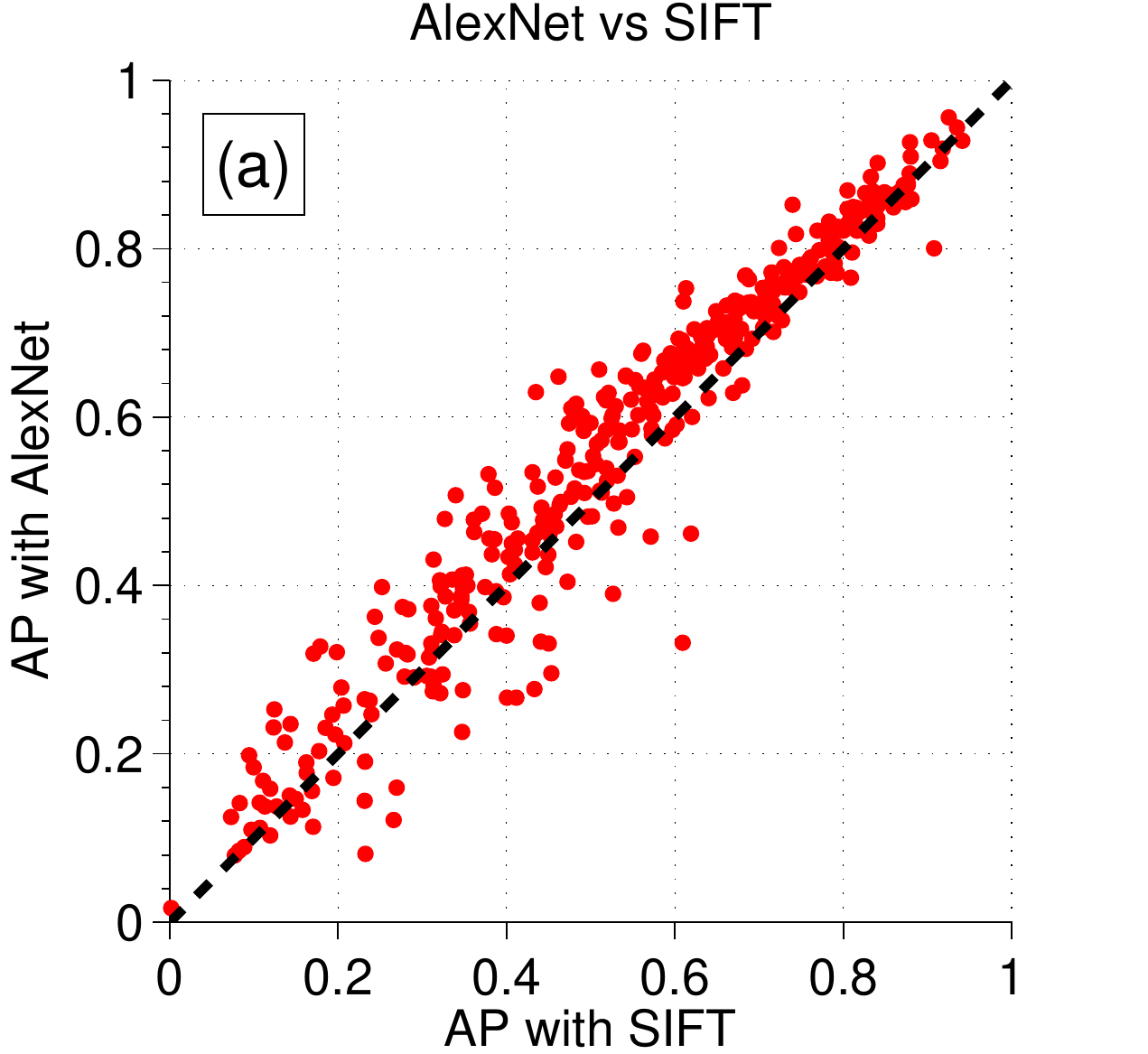} \hspace*{-20pt}&
  \includegraphics[width=.33\textwidth]{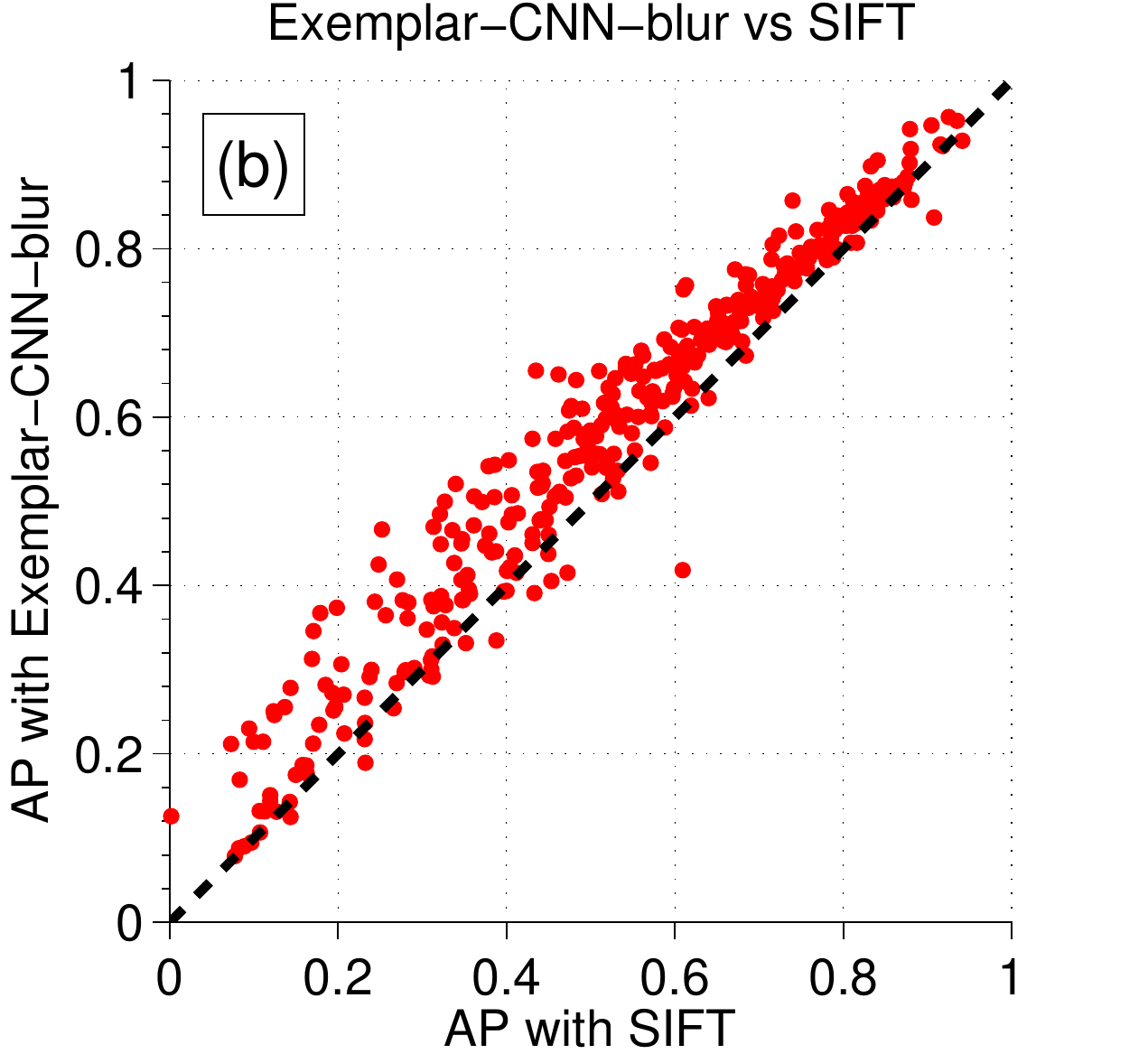} \hspace*{-20pt}&
  \includegraphics[width=.33\textwidth]{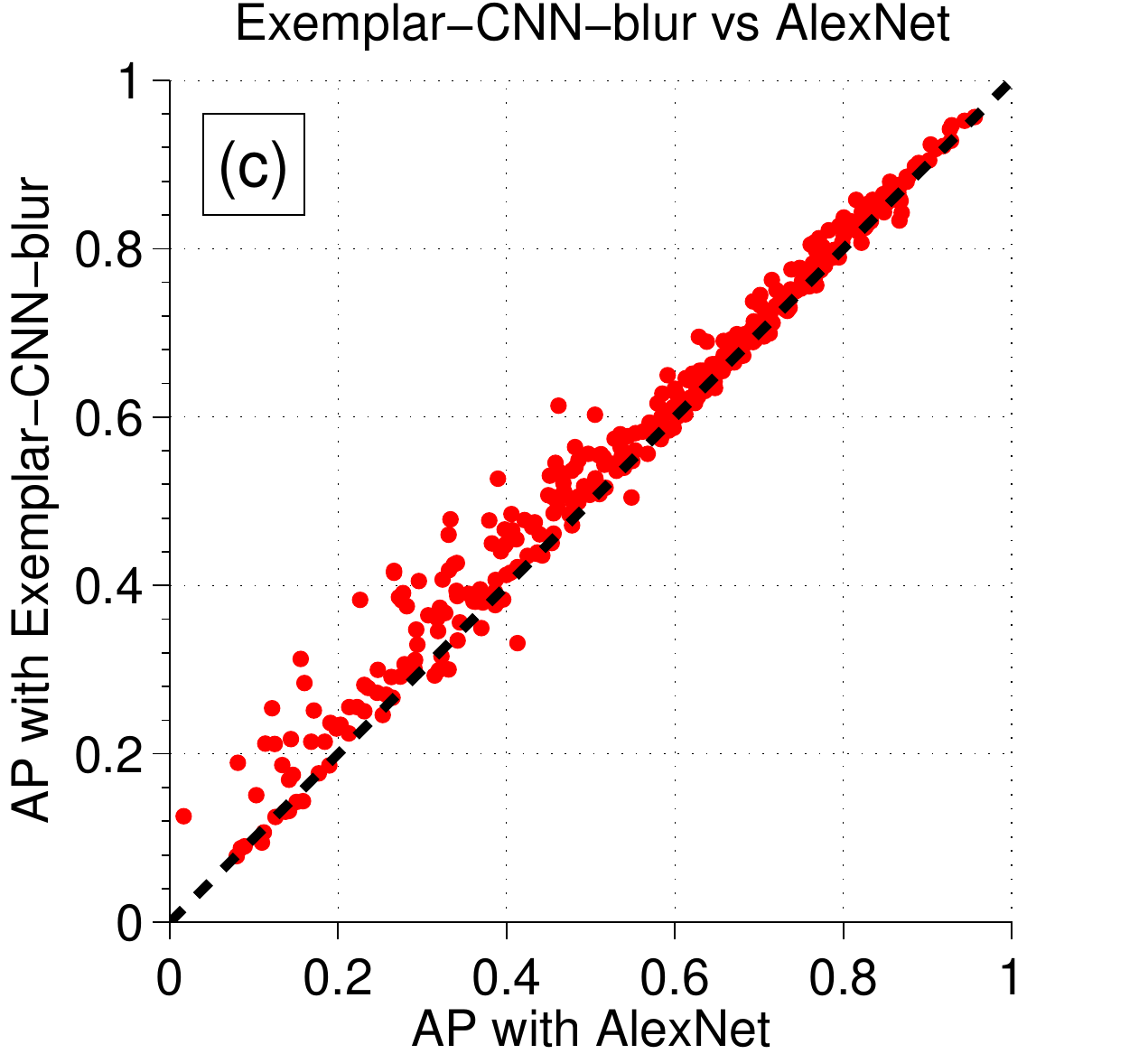} \\
  \includegraphics[width=.33\textwidth]{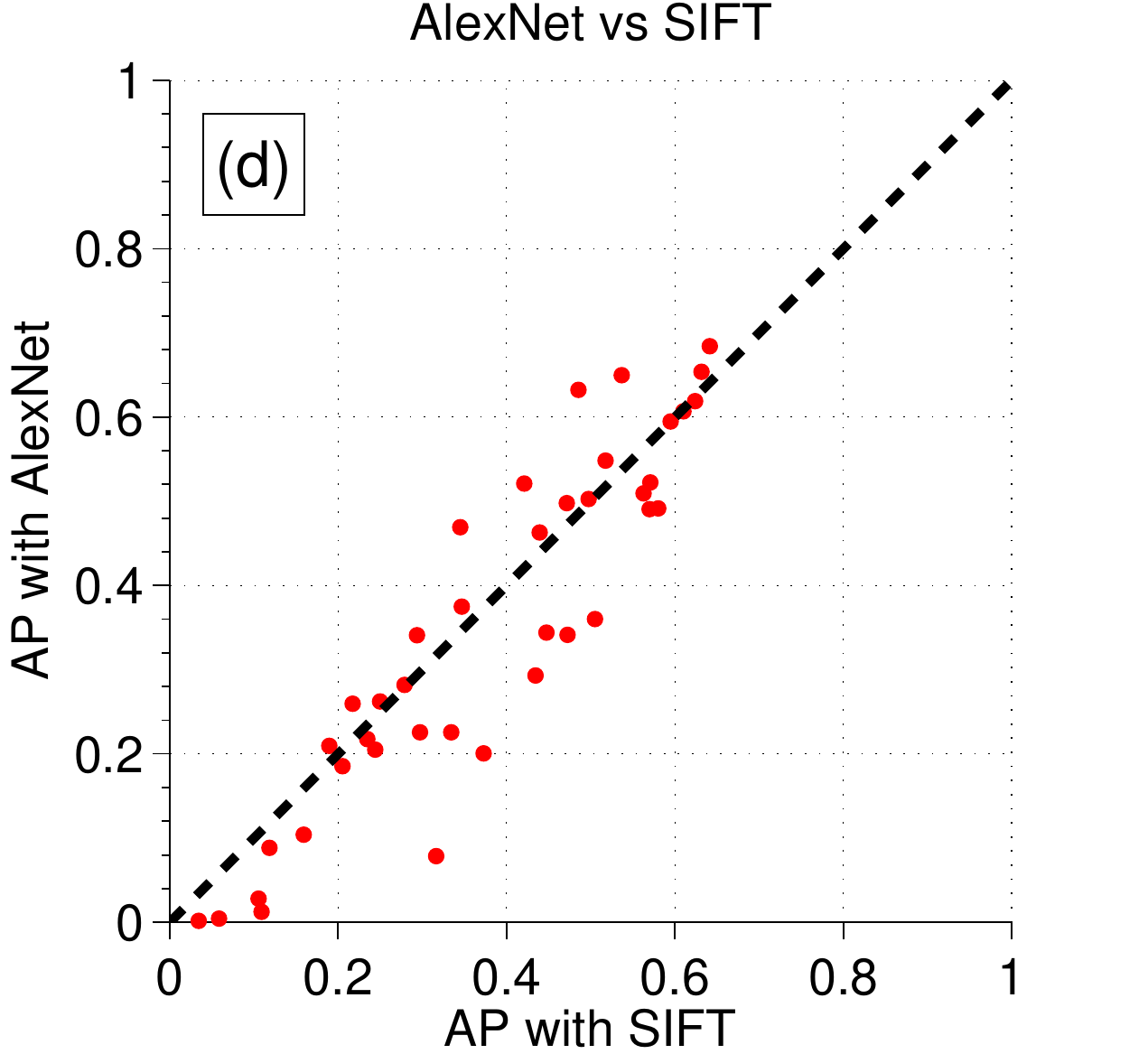} \vspace*{5 pt}\hspace*{-20pt} &
  \includegraphics[width=.33\textwidth]{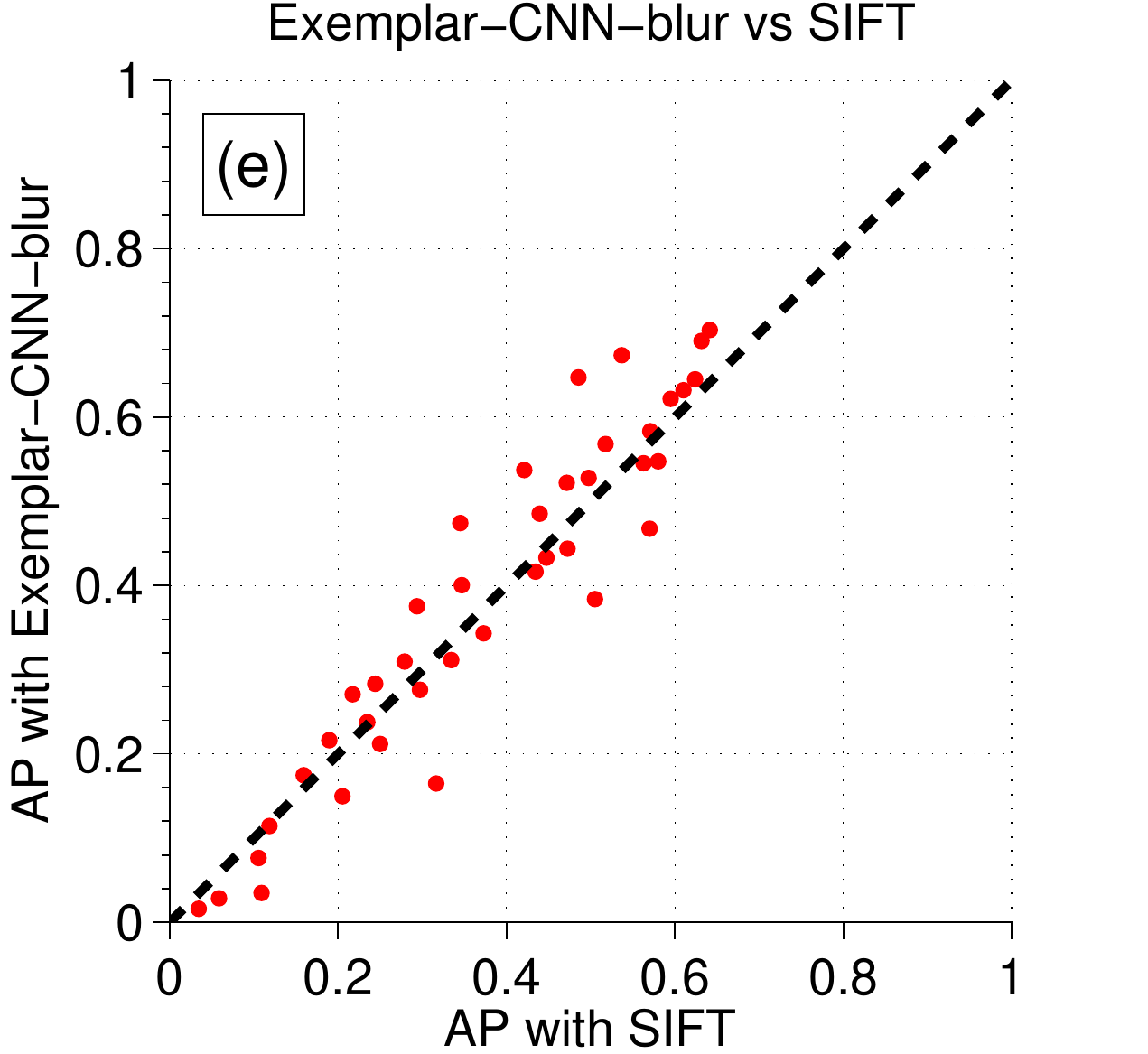} \hspace*{-20pt}&
  \includegraphics[width=.33\textwidth]{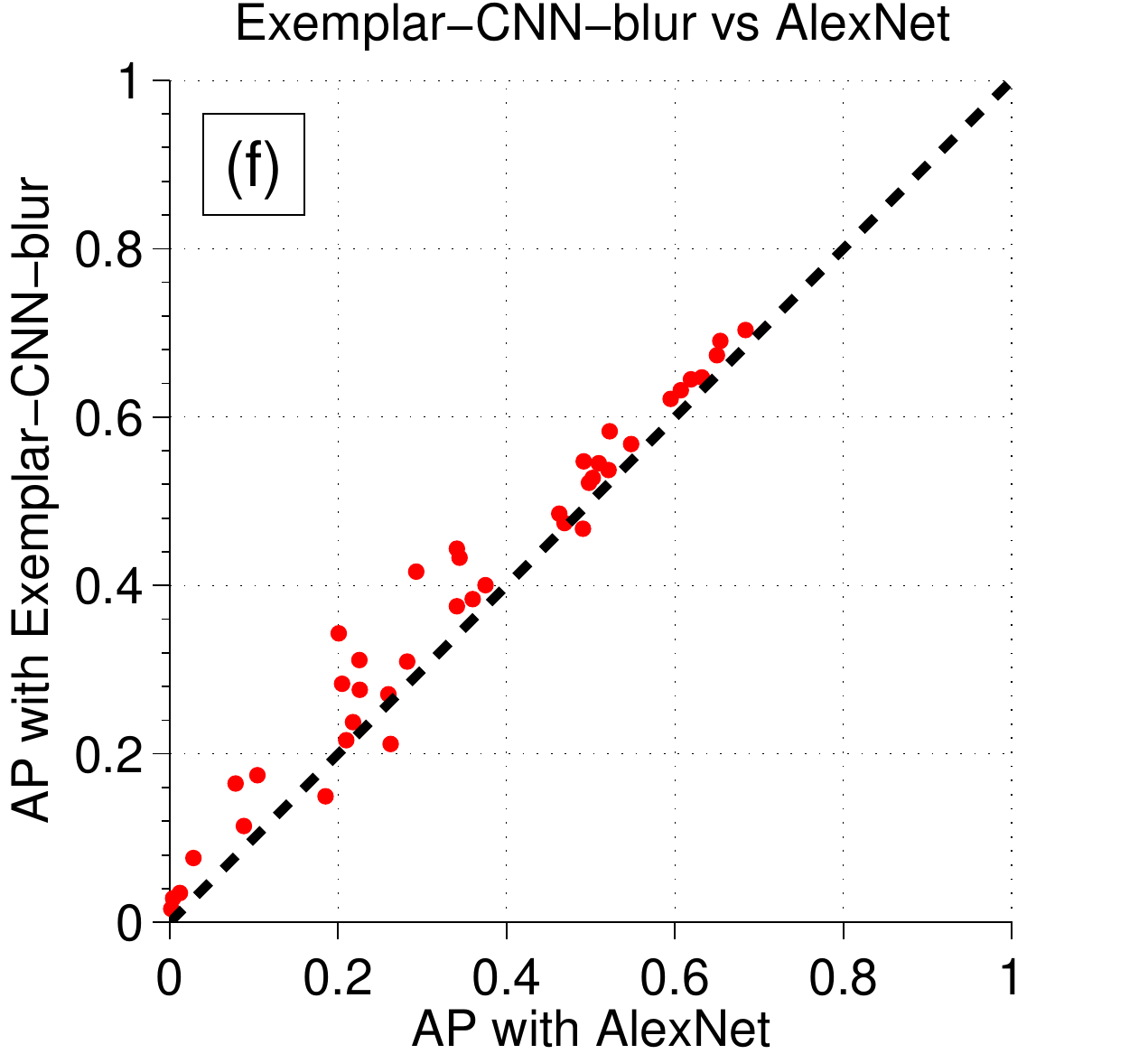}
\end{tabular}
\caption{Scatter plots for different pairs of descriptors on the \textbf{Flickr dataset (upper row)} and the \textbf{Mikolajczyk dataset (lower row)}. Each point in a scatter plot corresponds to one image pair, and its coordinates are the AP values obtained with the compared descriptors. AlexNet (supervised training) and the Exemplar-CNN yield features that outperform SIFT on most images of the Flickr dataset (a,b), but AlexNet is inferior to SIFT on the Mikolajczyk dataset. Features obtained with the unsupervised training procedure outperform the features from AlexNet on both datasets (c,f).}
\label{fig:scatterplots}
\end{figure*}

\subsection{Patch size and network layer}

The MSER detector returns ellipses of varying sizes, depending on the scale of the detected region. To compute descriptors from these elliptic regions we normalized the image patches to a fixed size. It is not immediately clear which patch size is best: larger patches provide a higher resolution, but enlarging them too much may introduce interpolation artifacts and the effect of high-frequency noise may be emphasized. Therefore, we optimized the patch size on the Flickr dataset for each method.

When using convolutional neural networks for region description, aside from the patch size there is another fundamental choice~-- the network layer from which the features are extracted. Features from higher layers are more abstract.

Fig.~\ref{fig:matching_patchsize} shows the average performance of each method when varying the patch size between $69$ and $157$. We chose the maximum patch size value such that most ellipses are smaller than that. We found that in case of SIFT, the performance monotonously grows and saturates at the maximum patch size. SIFT is based on normalized finite differences, and thus very robust to blurred edges caused by interpolation. In contrast, for the networks, especially for their lower layers, there is an optimal patch size, after which performance starts degrading. The lower network layers typically learn Gabor-like filters tuned to certain frequencies. Therefore, they suffer from over-smoothing caused by interpolation. Features from higher layers have access to larger receptive fields and, thus, can again benefit from larger patch sizes.

In the following experiments we used the optimal parameters given by Fig.~\ref{fig:matching_patchsize}: patch size $157$ for SIFT and $113$ for all other methods; layer $4$ for AlexNet and Exemplar-CNN-blur and layer $3$ for Exemplar-CNN-orig.

\begin{figure*}
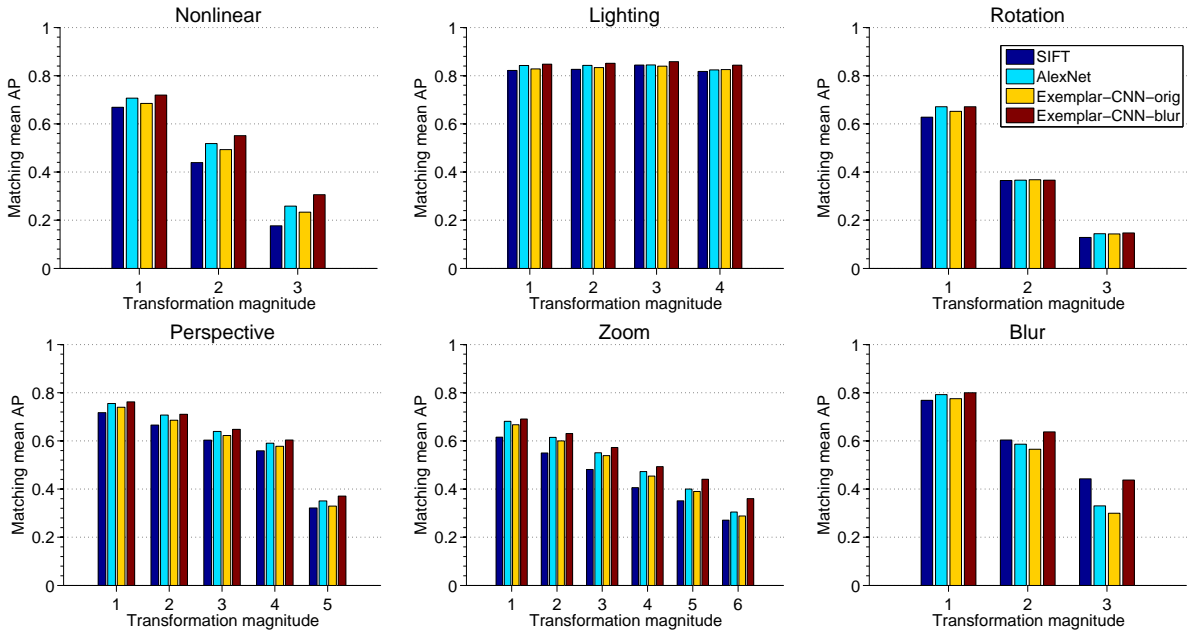

\begin{tabular}{ccc}
  \hspace*{-6 pt}\includegraphics[width=.3\textwidth]{nonlinear.pdf} &
  \hspace*{-11 pt}\includegraphics[width=.3\textwidth]{lighting.pdf} &
  \hspace*{-11 pt}\includegraphics[width=.3\textwidth]{rotation.pdf} \\ \vspace*{5 pt}
  \hspace*{-6 pt}\includegraphics[width=.3\textwidth]{perspective.pdf} &
  \hspace*{-11 pt}\includegraphics[width=.3\textwidth]{zoom.pdf} &
  \hspace*{-11 pt}\includegraphics[width=.3\textwidth]{blur.pdf} \\
\end{tabular}
\caption{Mean average precision on the \textbf{Flickr dataset} for various transformations. Except for the blur transformation, all networks perform consistently better than SIFT. The network trained with blur transformations can keep up with SIFT even on blur.}
\label{fig:training_matching}
\end{figure*}

\begin{figure*}
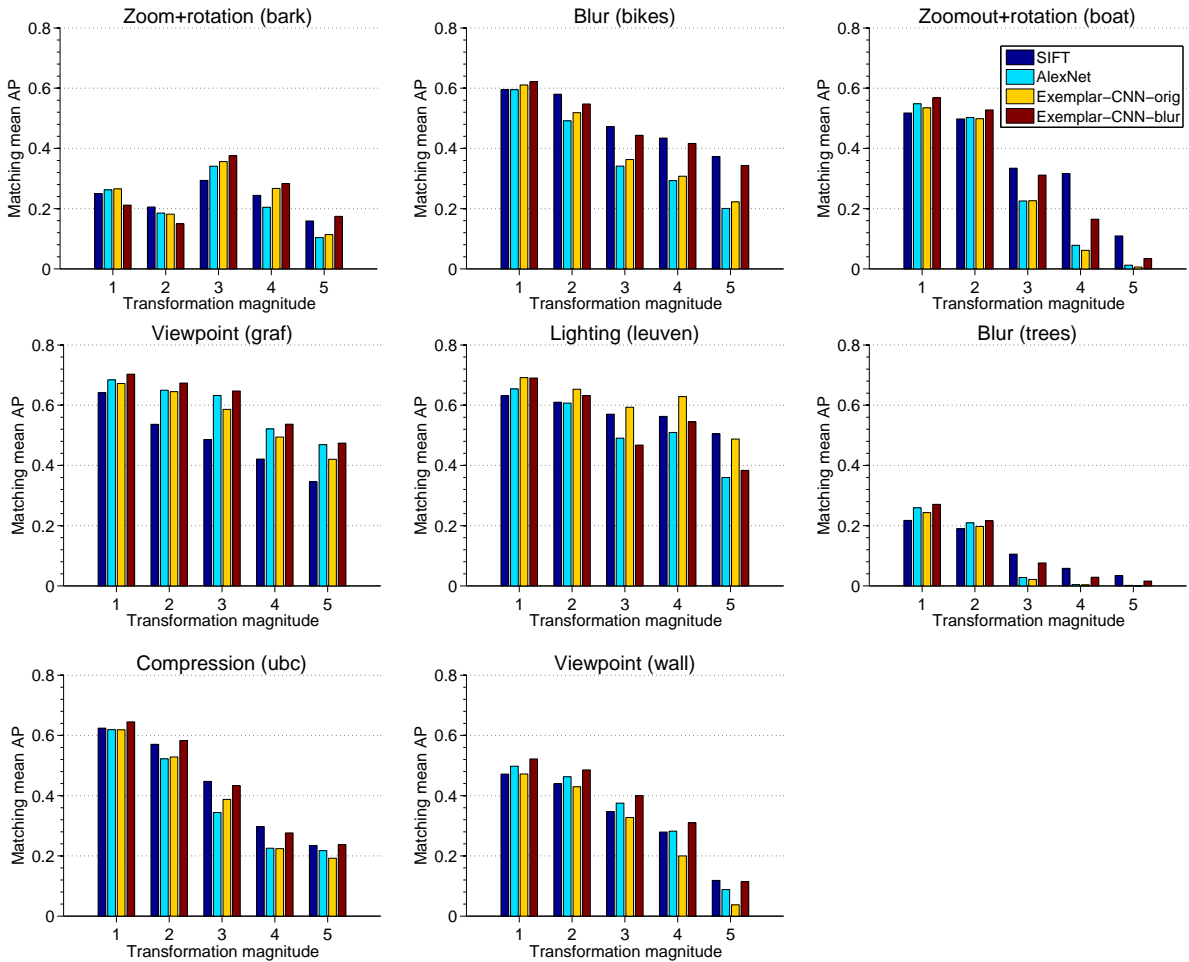

\begin{tabular}{ccc}
  \hspace*{-6 pt}\includegraphics[width=.3\textwidth]{miko1.pdf} &
  \hspace*{-11 pt}\includegraphics[width=.3\textwidth]{miko2.pdf} &
  \hspace*{-11 pt}\includegraphics[width=.3\textwidth]{miko3.pdf} \\ \vspace*{5 pt}
  \hspace*{-6 pt}\includegraphics[width=.3\textwidth]{miko4.pdf} &
  \hspace*{-11 pt}\includegraphics[width=.3\textwidth]{miko5.pdf} &
  \hspace*{-11 pt}\includegraphics[width=.3\textwidth]{miko6.pdf} \\ \vspace*{5 pt}
  \hspace*{-6 pt}\includegraphics[width=.3\textwidth]{miko7.pdf} &
  \hspace*{-11 pt}\includegraphics[width=.3\textwidth]{miko8.pdf} &
\end{tabular}
\caption{Mean average precision on the \textbf{Mikolajczyk dataset}. The networks perform better on viewpoint transformations, while SIFT is more robust to strong blur and lighting transformations.}
\label{fig:miko_matching}
\end{figure*}

\subsection{Results}

Fig.~\ref{fig:scatterplots} shows scatter plots that compare the performance of pairs of methods in terms of average precision. Each dot corresponds to an image pair. Points above the diagonal indicate better performance of the first method, and for points below the diagonal the AP of the second method is higher. The scatter plots also give an intuition of the variance in the performance difference.

Fig.~\ref{fig:scatterplots}a,b show that the features from both AlexNet and the Exemplar-CNN outperform SIFT on the Flickr dataset. However, especially for features from AlexNet there are some image pairs, for which SIFT performs clearly better. On the Mikolayczyk dataset, SIFT even outperforms features from AlexNet. We will analyze this in more detail in the next paragraph. Fig.~\ref{fig:scatterplots}c,f compare AlexNet with the Exemplar-CNN-blur and show that the loss function based on surrogate classes is superior to the loss function based on object class labels. In contrast to object classification, class-specific features are not advantageous for descriptor matching. A loss function that focuses on the invariance properties required for descriptor matching yields better results.

In Fig.~\ref{fig:training_matching} and~\ref{fig:miko_matching} we analyze the reason for the clearly inferior performance of AlexNet on some image pairs. The figures show the mean average precision on the various transformations of the datasets using the optimized parameters. On the Flickr dataset AlexNet performs better than SIFT for all transformations except blur, where there is a big drop in performance. Also on the Mikolayczyk dataset, the blur and zoomout transformations are the main reason for SIFT performing better overall. Actually this effect is not surprising. At the lower layers, the networks mostly contain filters that are tuned to certain frequencies. Also the features at higher layers seem to expect a certain sharpness for certain image structures. Consequently, a blurred version of the same image activates very different features. In contrast, SIFT is very robust to image blur as it uses simple finite differences that indicate edges at all frequencies, and the edge strength is normalized out.

The Exemplar-CNN-blur is much less affected by blur since it has learned to be robust to it. To demonstrate the importance of adding blur to the transformations, we also included the Exemplar-CNN which was used for the classification task, i.e., without blur among the transformations. Like AlexNet, it has problems with matching blurred images to the original image.

Computation times per image are shown in Table~\ref{tbl:comptime}. SIFT computation is clearly faster than feature computation by neural networks, but the computation times of the neural networks are not prohibitively large, especially when extracting many descriptors per image using parallel hardware.

\setlength{\tabcolsep}{4pt}
\begin{table}[h!]
%
%
%
\begin{center}
\begin{tabular}{l|c|c|c}
\hline\noalign{\smallskip}
\textbf{Method} & SIFT & AlexNet& Ex-CNN-blur \\
\noalign{\smallskip}
\hline
\noalign{\smallskip}
\textbf{CPU} & $\,4.5\mathrm{ms}$ & $\,28.2\mathrm{ms}$ & $\,103.9\mathrm{ms}$ \\
\noalign{\smallskip}
\hline
\noalign{\smallskip}
\textbf{GPU} & - & $\,0.7\mathrm{ms}$ & $\,1.8\mathrm{ms}$ \\
\noalign{\smallskip}
\hline
\noalign{\smallskip}
\end{tabular}
\caption{Feature computation times for a patch of 113 by 113 pixels.}
\label{tbl:comptime}
\end{center}
\end{table}

\section{Conclusions} \label{sect:discussion}

We have proposed a discriminative objective for unsupervised feature learning by training a CNN without object class labels. The core idea is to generate a set of surrogate labels via data augmentation, where the applied transformations define the invariance properties that are to be learned by the network. The learned features yield a large improvement in classification accuracy compared to features obtained with previous
unsupervised methods. These results strongly indicate that a discriminative objective is superior to objectives previously used for unsupervised feature learning. The unsupervised training procedure also lends itself to learn features for geometric matching tasks. A comparison to the long standing state-of-the-art descriptor for this task, SIFT, revealed a problem when matching neural network features in case of blur. We showed that by adding blur to the set of transformations applied during training, the features obtained with such a network are not much affected by this problem anymore and outperform SIFT on most image pairs. This simple inclusion of blur demonstrates the flexibility of the proposed unsupervised learning strategy. The strong relationship of the approach to data augmentation in supervised settings also emphasizes the value of data augmentation in general and suggests the use of more diverse transformations.

\appendices
\section{Formal analysis} \label{appendix:formal}
\begin{proposition} \label{prop:logexpsum}
The function
$$Z(\bx) = \log \norm{ \exp(\bx)}_1,\; \bx~\in~\oR^n$$
is convex. Moreover, for any $\bx \in \oR^n$ the kernel of its Hessian
matrix $\nabla^2 Z (\bx)$ is given by $\linspan (\bone)$
\end{proposition}
\begin{proof}
 Since
 \begin{equation}
   Z(\bx) = \log \norm{ \exp(\bx)}_1 =  \log \sum_{i=1}^n \exp(x_i)
 \end{equation}
 we need to prove the convexity of the log-sum-exp function.
 The Hessian $\nabla^2$ of this function is given as
 \begin{equation}
   \nabla^2 Z (\bx) = \frac{1}{(\bone^T \bu)^2} ((\bone^T \bu)\; \diag(\bu) - \bu \bu^T ),
 \label{eq:hessian}
 \end{equation}
 with $\bu = exp(\bx)$ and $\bone \in \oR^n$ being a vector of ones. To show the convexity we must prove  that  $\bz^T \nabla^2 Z(\bx)  \bz \geq 0$ for all $\bx, \bz \in \oR^n$. From
 \eqref{eq:hessian} we get
 \begin{multline}
  \bz^T\, \nabla^2 Z(\bx)\, \bz = \frac{1}{(\bone^T \bu)^2} ((\bone^T \bu)\; \bz^T \diag(\bu)\, \bz - \bz^T \bu \bu^T \bz ) \\
                 = \frac{(\sum_{k=1}^n u_k z_k^2) (\sum_{k=1}^n u_k) - (\sum_{k=1}^n u_k z_k)^2}{(\sum_{k=1}^n u_k)^2} \geq 0
 \label{eq:hessian_ineq}
 \end{multline}
since $(\sum_{k=1}^n u_k)^2 \geq 0$  and  $(\sum_{k=1}^n z_k u_k)^2 \leq (\sum_{k=1}^n u_k z_k^2)
(\sum_{k=1}^n u_k)$ due to the Cauchy-Schwarz inequality.

Inequality~\eqref{eq:hessian_ineq} only turns to equality if
\begin{equation}
\sqrt{u_k} z_k = c \sqrt{u_k},
\end{equation}
where the constant $c$ does not depend on $k$. This immediately gives $\bz = c \bone$, which proves the second statement of the proposition.
\end{proof}

\begin{proposition} \label{prop:jensens}
Let $\balpha \in \alphaset$ be a random vector with values in a bounded set $\alphaset \subset \oR^k$. Let $\bx(\cdot)\colon\, \alphaset \to \oR^n$ be a continuous function. Then inequality~\eqref{eq:jensens}
\begin{equation*}
 \me_\balpha \left[ \log \norm{\exp (\bx(\balpha))}_1 \right] - \log
 \norm{\exp (\me_\balpha [\bx(\balpha)])}_1 \geq 0
\end{equation*}
holds and only turns to equality if for all $\balpha_1, \balpha_2 \in \alphaset$: $(\bx(\balpha_1) - \bx(\balpha_2)) \in \linspan (\bone)$ .
\end{proposition}
\begin{proof}
Inequality~\eqref{eq:jensens} immediately follows from convexity of the function $\log \norm{\exp (\cdot)}_1$ and Jensen's inequality.

Jensen's inequality only turns to equality if the function it is applied to is affine-linear on the convex hull of the integration region. In particular this implies
\begin{equation}
 (\bx(\balpha_1) - \bx(\balpha_2))^T\; \nabla^2 Z (\bx(\balpha_1))\; (\bx(\balpha_1) - \bx(\balpha_2)) = 0
\end{equation}
for all $\balpha_1, \balpha_2 \in \alphaset$. The second statement of Proposition~\ref{prop:logexpsum} thus immediately gives $\bx(\balpha_1) - \bx(\balpha_2) = c \bone$, Q.E.D.

\end{proof}

\section{Method details} 

We describe here in detail the network architectures we evaluated and explain the network training procedure.
We also provide details of the clustering process we used to improve Exemplar-CNN.

\subsection{Network Architecture} \label{appendix:network_architecture}
We tested various network architectures in combination with our
training procedure. They are
coded as follows: NcF stands for a convolutional layer with $N$ filters
of size $F \times F$ pixels, Nf stands for a fully connected layer with
$N$ units. For example, 64c5-64c5-128f denotes a network with
two convolutional layers containing 64 filters spanning $5 \times 5$
pixels each followed by a fully connected layer with $128$
units. The last specified layer is always succeeded by a softmax
layer, which serves as the network output. We applied
$2 \times 2$ max-pooling to the outputs of the first and second
convolutional layers.

As stated in the paper we used a 64c5-64c5-128f architecture in our
experiments to evaluate
the influence of different components of the augmentation procedure  (we refer to this architecture as the 'small' network).
A large network, coded as 64c5-128c5-256c5-512f, was then used to
achieve better classification performance.

All considered networks contained rectified linear units in each layer
but the softmax layer. Dropout was applied to the fully connected
layer.

\subsection{Training the Networks} \label{appendix:training_details}

We adopted the common practice of training the network with stochastic gradient descent with a fixed momentum of $0.9$.
We started with a learning rate of $0.01$ and gradually decreased the
learning rate during training.  That is, we trained until there was no
improvement in validation error, then decreased the learning rate by a
factor of $3$, and repeated this procedure until convergence. Training times on a Titan GPU were roughly $1.5$ days for the 64c5-64c5-128f network, $4$ days for the 64c5-128c5-256c5-512f network and $9$ days for the 92c5-256c5-512c5-1024f network.

\subsection{Clustering} \label{app:clustering}
To judge about similarity of the clusters we use the following simple heuristics. The method of~\cite{Singh_ECCV2012} gives us a set of linear SVMs. We apply these SVMs to the whole STL-10 unlabeled dataset and select $N_{percluster}=10$ top firing images per SVM, which gives us a set of initial clusters. We then compute the overlap (number of common images) of each pair of these clusters. We set two thresholds $T_{merge}=3$ and $T_{discard}=1$ and perform a greedy procedure: starting from the most overlapping pair of clusters, we merge the clusters if their overlap exceeds $T_{merge}$ and discard one of the clusters if the overlap is between $T_{discard}$ and $T_{merge}$.

\section{Details of computing the measure of invariance} \label{appendix:invariance_details}

We now explain in detail and motivate the computation of the normalized Euclidean distance used as a measure of invariance in the paper.

First we compute feature vectors of all image patches and their transformed versions. Then we normalize each feature vector to unit Euclidean norm and compute the Euclidean distances between each original patch and all of its transformed versions. For each transformation and magnitude we average these distances over all patches. Finally, we divide the resulting curves by their maximal values (typically it is the value for the maximum magnitude of the transformation).

The normalizations are performed to compensate for possibly different scales of different features. Normalizing feature vectors to unit length ensures that the values are in the same range for different features. The final normalization of the curves by the maximal value allows to compensate for different variation of different features: as an extreme, a constant feature would be considered perfectly invariant without this normalization, which is certainly not desirable.

The resulting curves show how quickly the feature representation changes when an image is transformed more and more. A representation for which the curve steeply goes up and then remains constant cannot be considered invariant to the transformation: the feature vector of the transformed patch becomes completely uncorrelated with the original feature vector even for small magnitudes of the transformation. On the other hand, if the curve grows gradually, this indicates that the feature representation changes slowly when the transformation is applied, meaning invariance or, rather, covariance of the representation.

\ifCLASSOPTIONcompsoc
  \section*{Acknowledgments}
\else
  \section*{Acknowledgment}
\fi

AD, PF, and TB acknowledge funding by the ERC Starting Grant VideoLearn (279401). JTS and MR are supported by the BrainLinks-BrainTools Cluster of Excellence funded by the German Research Foundation (EXC 1086). PF acknowledges a fellowship by the Deutsche Telekom Stifung.

\ifCLASSOPTIONcaptionsoff
  \newpage
\fi



%
{\small
\bibliographystyle{IEEEtran}
\bibliography{./dosovits_new,./matchingbib}

\begin{thebibliography}{10}
\providecommand{\url}[1]{#1}
\csname url@samestyle\endcsname
\providecommand{\newblock}{\relax}
\providecommand{\bibinfo}[2]{#2}
\providecommand{\BIBentrySTDinterwordspacing}{\spaceskip=0pt\relax}
\providecommand{\BIBentryALTinterwordstretchfactor}{4}
\providecommand{\BIBentryALTinterwordspacing}{\spaceskip=\fontdimen2\font plus
\BIBentryALTinterwordstretchfactor\fontdimen3\font minus
  \fontdimen4\font\relax}
\providecommand{\BIBforeignlanguage}[2]{{%
\expandafter\ifx\csname l@#1\endcsname\relax
\typeout{** WARNING: IEEEtran.bst: No hyphenation pattern has been}%
\typeout{** loaded for the language `#1'. Using the pattern for}%
\typeout{** the default language instead.}%
\else
\language=\csname l@#1\endcsname
\fi
#2}}
\providecommand{\BIBdecl}{\relax}
\BIBdecl

\bibitem{Krizhevsky_NIPS2012}
A.~Krizhevsky, I.~Sutskever, and G.~E. Hinton, ``{ImageNet} classification with
  deep convolutional neural networks,'' in \emph{NIPS}, 2012, pp. 1106--1114.

\bibitem{Zeiler_ECCV2014}
M.~D. Zeiler and R.~Fergus, ``Visualizing and understanding convolutional
  networks,'' in \emph{{ECCV}}, 2014.

\bibitem{Donahue_ICML2014}
J.~Donahue, Y.~Jia, O.~Vinyals, J.~Hoffman, N.~Zhang, E.~Tzeng, and T.~Darrell,
  ``{DeCAF}: {A} deep convolutional activation feature for generic visual
  recognition,'' in \emph{{ICML}}, 2014.

\bibitem{Razavian_CVPR2014}
A.~S. Razavian, H.~Azizpour, J.~Sullivan, and S.~Carlsson, ``{CNN} features
  off-the-shelf: An astounding baseline for recognition,'' in \emph{{CVPR}
  Workshops 2014}, 2014, pp. 512--519.

\bibitem{Girshick_CVPR2014}
R.~Girshick, J.~Donahue, T.~Darrell, and J.~Malik, ``Rich feature hierarchies
  for accurate object detection and semantic segmentation,'' in \emph{{CVPR}},
  2014.

\bibitem{Sermanet_ICLR2014}
P.~Sermanet, D.~Eigen, X.~Zhang, M.~Mathieu, R.~Fergus, and Y.~LeCun,
  ``{OverFeat}: Integrated recognition, localization and detection using
  convolutional networks.'' in \emph{ICLR}, 2014.

\bibitem{Hariharan_CVPR2015}
B.~Hariharan, P.~Arbeláez, R.~Girshick, and J.~Malik, ``Hypercolumns for
  object segmentation and fine-grained localization,'' \emph{CVPR}, 2015.

\bibitem{Long_CVPR2015}
J.~Long, E.~Shelhamer, and T.~Darrell, ``Fully convolutional networks for
  semantic segmentation,'' in \emph{CVPR}, 2015.

\bibitem{imagenet}
J.~Deng, W.~Dong, R.~Socher, L.-J. Li, K.~Li, and L.~Fei-Fei, ``{ImageNet: A
  Large-Scale Hierarchical Image Database},'' in \emph{CVPR}, 2009.

\bibitem{pascal}
M.~Everingham, L.~Gool, C.~K. Williams, J.~Winn, and A.~Zisserman, ``{The
  Pascal Visual Object Classes (VOC) Challenge},'' \emph{IJCV}, vol.~88, no.~2,
  pp. 303--338, 2010.

\bibitem{LeCun_NC1989}
Y.~LeCun, B.~Boser, J.~S. Denker, D.~Henderson, R.~E. Howard, W.~Hubbard, and
  L.~D. Jackel, ``Backpropagation applied to handwritten zip code
  recognition,'' \emph{Neural Computation}, vol.~1, no.~4, pp. 541--551, 1989.

\bibitem{Kavukcuoglu_NIPS2010}
K.~Kavukcuoglu, P.~Sermanet, Y.~Boureau, K.~Gregor, M.~Mathieu, and Y.~LeCun,
  ``Learning convolutional feature hierachies for visual recognition,'' in
  \emph{NIPS}, 2010.

\bibitem{Vincent_ICML2008}
P.~Vincent, H.~Larochelle, Y.~Bengio, and P.-A. Manzagol, ``Extracting and
  composing robust features with denoising autoencoders,'' in \emph{ICML},
  2008, pp. 1096--1103.

\bibitem{Zou_NIPS2012}
W.~Y. Zou, A.~Y. Ng, S.~Zhu, and K.~Yu, ``Deep learning of invariant features
  via simulated fixations in video,'' in \emph{NIPS}, 2012, pp. 3212--3220.

\bibitem{Sohn_ICML2012}
K.~Sohn and H.~Lee, ``Learning invariant representations with local
  transformations,'' in \emph{ICML}, 2012.

\bibitem{Hui_ICML2013}
K.~Y. Hui, ``Direct modeling of complex invariances for visual object
  features,'' in \emph{ICML}, 2013.

\bibitem{Simard_NIPS1992}
P.~Simard, B.~Victorri, Y.~LeCun, and J.~S. Denker, ``{Tangent Prop} - {A}
  formalism for specifying selected invariances in an adaptive network,'' in
  \emph{NIPS}, 1992.

\bibitem{Drucker_TNN1992}
H.~Drucker and Y.~LeCun, ``Improving generalization performance using double
  backpropagation,'' \emph{IEEE Transactions on Neural Networks}, vol.~3,
  no.~6, pp. 991--997, 1992.

\bibitem{Amini_ECAI2002}
M.-R. Amini and P.~Gallinari, ``Semi supervised logistic regression,'' in
  \emph{ECAI}, 2002, pp. 390--394.

\bibitem{Grandvalet_SSL2006}
Y.~Grandvalet and Y.~Bengio, ``Entropy regularization,'' in
  \emph{Semi-Supervised Learning}.\hskip 1em plus 0.5em minus 0.4em\relax {MIT}
  Press, 2006, pp. 151--168.

\bibitem{Ahmed_ECCV2008}
A.~Ahmed, K.~Yu, W.~Xu, Y.~Gong, and E.~Xing, ``Training hierarchical
  feed-forward visual recognition models using transfer learning from
  pseudo-tasks.'' in \emph{ECCV (3)}, 2008, pp. 69--82.

\bibitem{Collobert_JMLR2011}
R.~Collobert, J.~Weston, L.~Bottou, M.~Karlen, K.~Kavukcuoglu, and P.~Kuksa,
  ``Natural language processing (almost) from scratch,'' \emph{Journal of
  Machine Learning Research}, vol.~12, pp. 2493--2537, 2011.

\bibitem{Wagner_NIPS2013}
S.~Wager, S.~Wang, and P.~Liang, ``Dropout training as adaptive
  regularization,'' in \emph{NIPS}, 2013.

\bibitem{Rifai_NIPS2011}
S.~Rifai, Y.~N. Dauphin, P.~Vincent, Y.~Bengio, and X.~Muller, ``The manifold
  tangent classifier,'' in \emph{NIPS}, 2011.

\bibitem{Coates_AISTATS2010}
A.~Coates, H.~Lee, and A.~Y. Ng, ``{An analysis of single-layer networks in
  unsupervised feature learning},'' \emph{AISTATS}, 2011.

\bibitem{Krizhevsky_thesis2009}
A.~Krizhevsky and G.~Hinton, ``Learning multiple layers of features from tiny
  images,'' \emph{Master's thesis, Department of Computer Science, University
  of Toronto}, 2009.

\bibitem{FeiFei_CVPR2004}
L.~Fei-Fei, R.~Fergus, and P.~Perona, ``Learning generative visual models from
  few training examples: {A}n incremental bayesian approach tested on 101
  object categories,'' in \emph{CVPR WGMBV}, 2004.

\bibitem{Griffin_2007}
G.~Griffin, A.~Holub, and P.~Perona, ``Caltech-256 object category dataset,''
  California Institute of Technology, Tech. Rep. 7694, 2007.

\bibitem{Hinton_arxiv2012}
G.~E. Hinton, N.~Srivastava, A.~Krizhevsky, I.~Sutskever, and R.~R.
  Salakhutdinov, ``Improving neural networks by preventing co-adaptation of
  feature detectors,'' 2012, pre-print, arxiv:cs/1207.0580v3.

\bibitem{Srivastava_JMLR2014}
N.~Srivastava, G.~Hinton, A.~Krizhevsky, I.~Sutskever, and R.~Salakhutdinov,
  ``Dropout: A simple way to prevent neural networks from overfitting,''
  \emph{Journal of Machine Learning Research}, vol.~15, pp. 1929--1958, 2014.

\bibitem{caffe}
Y.~Jia, E.~Shelhamer, J.~Donahue, S.~Karayev, J.~Long, R.~Girshick,
  S.~Guadarrama, and T.~Darrell, ``Caffe: Convolutional architecture for fast
  feature embedding,'' \emph{arXiv preprint arXiv:1408.5093}, 2014.

\bibitem{Coates_NIPS2011}
A.~Coates and A.~Y. Ng, ``Selecting receptive fields in deep networks,'' in
  \emph{NIPS}, 2011, pp. 2528--2536.

\bibitem{Boureau_ICCV2011}
Y.~Boureau, N.~{Le Roux}, F.~Bach, J.~Ponce, and Y.~LeCun, ``Ask the locals:
  multi-way local pooling for image recognition,'' in \emph{ICCV'11}.\hskip 1em
  plus 0.5em minus 0.4em\relax IEEE, 2011.

\bibitem{Bo_ISER2012}
L.~Bo, X.~Ren, and D.~Fox, ``Unsupervised feature learning for {RGB-D} based
  object recognition,'' in \emph{ISER}, June 2012.

\bibitem{Bo_CVPR2013}
------, ``Multipath sparse coding using hierarchical matching pursuit,'' in
  \emph{CVPR}, 2013, pp. 660--667.

\bibitem{Swersky_NIPS2013}
K.~Swersky, J.~Snoek, and R.~P. Adams, ``Multi-task bayesian optimization,'' in
  \emph{{NIPS}}, 2013.

\bibitem{Lee_NIPS2014}
C.-Y. Lee, S.~Xie, P.~Gallagher, Z.~Zhang, and Z.~Tu, ``Deeply supervised
  nets,'' in \emph{Deep Learning and Representation Learning Workshop, NIPS},
  2014.

\bibitem{He_ECCV2014}
K.~He, X.~Zhang, S.~Ren, and J.~Sun, ``Spatial pyramid pooling in deep
  convolutional networks for visual recognition,'' in \emph{{ECCV}}, 2014.

\bibitem{Singh_ECCV2012}
S.~Singh, A.~Gupta, and A.~A. Efros, ``Unsupervised discovery of mid-level
  discriminative patches,'' in \emph{ECCV}, 2012.

\bibitem{Lo04}
D.~G. Lowe, ``Distinctive image features from scale-invariant keypoints,''
  \emph{IJCV}, vol.~60, no.~2, pp. 91--110, Nov. 2004.

\bibitem{mikopami05}
K.~Mikolajczyk and C.~Schmid, ``A performance evaluation of local
  descriptors,'' \emph{IEEE Trans. Pattern Anal. Mach. Intell.}, vol.~27,
  no.~10, pp. 1615--1630, 2005.

\bibitem{mikoijcv05}
K.~Mikolajczyk, T.~Tuytelaars, C.~Schmid, A.~Zisserman, J.~Matas,
  F.~Schaffalitzky, T.~Kadir, and L.~J.~V. Gool, ``A comparison of affine
  region detectors,'' \emph{IJCV}, vol.~65, no. 1-2, pp. 43--72, 2005.

\bibitem{mser}
J.~Matas, O.~Chum, M.~Urban, and T.~Pajdla, ``Robust wide baseline stereo from
  maximally stable extremal regions,'' in \emph{Proc. BMVC}, 2002, pp.
  36.1--36.10, doi:10.5244/C.16.36.

\end{thebibliography}
}






\end{document}